%% file: main.tex
\documentclass{article} 
\usepackage{iclr2027/iclr2027_conference,times}
\usepackage{amsthm, amssymb}
\usepackage{algorithm}
\usepackage{algpseudocode}
\usepackage{mathtools} 
\usepackage{subcaption}
\usepackage{thmtools}
\usepackage{thm-restate}
\usepackage{multirow}
\input{iclr2027/math_commands}

\usepackage{booktabs}
\usepackage{hyperref}
\usepackage{url}
\usepackage{float}

\theoremstyle{definition}
\newtheorem{definition}{Definition}
\newtheorem{condition}{Condition}
 
\theoremstyle{plain}
\newtheorem{lemma}{Lemma}

\theoremstyle{remark}
\newtheorem{remark}{Remark}

\title{Adaptive Computation with Conformal  Certification}

\author{
    Shehtab Zaman$^\dagger$, Aneesh Ragnekar$^{\dagger}\thanks{Work done while at MSKCC.}$,  Nancy Lee$^{\ddagger}$, \& Harini Veeraraghavan$^\dagger$\\
    ${^\dagger}$Department of Medical Physics, $^{\ddagger}$ Department of Radiation Oncology\\
Memorial Sloan Kettering Cancer Center\\
New York, NY, USA \\
}

\newcommand{\cX}{\mathcal{X}}
\newcommand{\cY}{\mathcal{Y}}
\newcommand{\cG}{\mathcal{G}}
\newcommand{\OO}{\Omega}

\newcommand{\coursePart}{\Omega_{\text{coarse}}}
\newcommand{\finePart}{\Omega_{\text{fine}}}

\iclrfinalcopy 
\begin{document}

\maketitle
\begin{abstract}
Scientific imaging data involving medical (e.g., 3D radiographic, microscopic images) and astronomical sensors pose different challenges from natural images. 
One key challenge is that these datasets are often orders of magnitude larger than natural images. 
Efficient training and inference on such data is challenging due to the increased hardware requirements, often needing bespoke solutions. 
Moreover, the relevant objects such as tumors and astronomical phenomena are sparsely distributed with respect to the background information. 
Standard efficiency techniques like quantization and pruning apply uniform scaling of all regions, potentially compromising the precision/fidelity of predictions. 
Thus, efficient methods that reduce computational requirements without losing resolution for relevant parts of the data and with uncertainty estimates on predictions are essential.
We present \textit{Conformal Adaptation} with risk guarantees for uncertainty-aware sparse computations. Our formalism is general, spanning dense pixel-level classification and regression tasks where hierarchical models that predict at low resolution, certify the signal using conformal risk control, and refine the certified predictions. 
We benchmark our method on multiple tasks such as 3D medical image segmentation, 2D galaxy flux prediction, and 1D calcium spike rate prediction, achieving $2.5$--$72\times$ wall-clock speedup across the various tasks, including on commodity CPUs. 
\end{abstract}

\section{Introduction}

Deep learning models have expanded far beyond benchmark vision and natural language processing tasks, becoming foundational components of high-throughput scientific pipelines ranging from drug discovery and weather forecasting to medical imaging and astronomical surveys. However, deploying these models in real-world environments frequently exposes a fundamental tension between predictive capacity and resource constraints~\citep{jazbec_fast_2024}. While frontier models deliver state-of-the-art accuracy, their practical utility rapidly degrades when constrained by edge hardware, strict latency budgets, or extreme data scales. As computational access becomes a bottleneck across disciplines, designing models that are efficient, scalable, and trustworthy has become a central challenge.

This bottleneck also carries severe environmental and structural costs. Maximizing benchmark performance via scale demands unsustainable energy and carbon footprints~\citep{schwartz2020green, mehlin2023towards}. It also raises steep capital barriers, concentrating research within well-funded enterprise labs while sidelining academic and clinical researchers~\citep{guan2024energy}. Because hardware scaling yields diminishing returns, algorithmic efficiency is critical for responsible scientific computing.

Practitioners pursue efficiency along two broad routes. Hardware specialization and kernel-level tuning preserve computational fidelity but require long development cycles and capital-intensive infrastructure. 
Algorithmic optimizations such as post-training quantization~\citep{liu_llm-qat:_2024, chen_efficientqat:_2025}, structural pruning~\citep{pruning-1, pruning-2}, low-rank factorization~\citep{lora,qlora}, and distillation~\citep{knowledge} trade some precision for large runtime reductions on commodity hardware. 
On standard vision and NLP benchmarks this tradeoff is usually benign, and sub-8-bit representations are now common practice. 
These techniques are increasingly applied in safety-critical settings~\citep{slimUNETRV2, tokenseg, latup}, where they are typically validated by aggregate metrics but can still degrade performance~\citep{alcaraz2026cloud}. 
Many dense scientific tasks, however, are characterized by extreme spatial sparsity: tumors in a positron emission tomography (PET) volume or rare transients in a sky survey occupy a tiny fraction of the data. 
In this regime an aggregate metric can remain high while revealing little about the few regions that matter, and the cost of an error (a missed lesion, a false discovery) is concentrated there. 
Practitioners therefore benefit from acceleration methods whose risk on the sparse signal is explicitly bounded.



We approach efficiency along a complementary axis. Rather than reducing the cost of each evaluation, we reduce the number of locations that require expensive evaluation and certify that decision. We introduce \textit{Conformal Adaptation}, a framework for allocating compute across dense prediction tasks. A low-resolution pass identifies candidate regions of interest, and conformal risk control~\citep{CRC} calibrates this selection to provide finite-sample, distribution-free bounds on the expected fraction of signal missed, at a user-specified level. A hierarchy of sparse operators then refines only the selected regions during both training and inference, and we show that the calibrated bound is preserved through refinement under two stated conditions. As the coarse screening model is an ordinary network, \textit{Conformal Adaptation} composes naturally with quantization, pruning, and other model-level compression techniques.


\paragraph{Contributions}
 Our primary contributions are \textbf{(a)} an end-to-end framework combining conformalized candidate selection with sparse execution, delivering significant training and inference speedups backed by conformal risk control (CRC) guarantees, \textbf{(b)} a general formulation of admissible label aggregation operators that makes the framework applicable across tasks, and \textbf{(c)} case studies spanning multiple scientific domains, data modalities, and sparsity regimes that empirically validate the guarantees and quantify the compute savings. 

The manuscript is organized as follows. We first establish common notation for dense predictions, sparse refinement solutions, and CRC in Sec.~\ref{sec:background}. We then detail the \textit{Conformal Adaptation} algorithm in Sec.~\ref{sec:con-adapt} and define our performance metrics in Sec.~\ref{sec:metrics}. We showcase the applicability of our method across segmentation (Sec.~\ref{sec:seg}) and regression tasks (Sec.~\ref{sec:reg}), review related work in Sec.~\ref{sec:related-work}, and discuss limitations, trade-offs, and future work in Sec.~\ref{sec:limitations}. 

\section{Background}
\label{sec:background}

\paragraph{Notation. } We focus on dense prediction tasks whose domains admit
\emph{hierarchical} partitionings. Inputs $X$ and outputs $Y$ are spatially discretized, with $X\in\mathbb{R}^{|\OO|\times l}$ and $Y\in\mathbb{R}^{|\OO|\times k}$, where $\OO$ is a finite set of \emph{sites} (pixels in 2D, voxels in 3D, or more generally any indexed cell) and $l$ and $k$ are the numbers of input and output channels, respectively. We say a coarse partitioning $\coursePart$
\emph{coarsens} a fine partitioning $\finePart$, written $\coursePart\preceq\finePart$, when
every cell of $\coursePart$ is a disjoint union of cells of $\finePart$; equivalently,
the fine partitioning refines the coarse one, and each fine site lies in
exactly one coarse cell.

We seek to learn a non-linear function $\mathcal{F}_{\theta}: \mathcal{X} \rightarrow \mathcal{Y}$.
Due to hierarchical partitioning, we consider $\mathcal{F}_\theta$ that decomposes into a coarse model followed by a refinement operator, $\mathcal{F}_\theta(X) = \cG\bigl(X^{\finePart}, \mathcal{F}_{\coursePart}(X^{\coursePart})\bigr)$, so that predictions can be performed at multiple levels.
It is far cheaper to predict on $\coursePart$ and refine than to predict directly on $\finePart$ with a full-resolution model $\mathcal{F}_{\finePart}$.

\paragraph{Conformal Risk Control.}
We turn to CRC~\citep{CRC} as a distribution-free black-box method for providing prediction guarantees at coarse resolution.

Given a per-image loss $L(\lambda;X,Y)$ that, for every fixed $(X,Y)$, is non-increasing and right-continuous in $\lambda$, bounded in $[0,B]$, and satisfies $L(1;X,Y)\le\alpha$, together with calibration data
$(X_i,Y_i)_{i=1}^n$ exchangeable with the test point $(X_{n+1},Y_{n+1})$, the threshold

\begin{equation}
  \hat\lambda \;=\; \inf\Bigl\{\lambda :
  \tfrac{1}{n+1}\bigl(\textstyle\sum_{i=1}^n L(\lambda;X_i,Y_i) + B\bigr)
  \le \alpha \Bigr\}
  \label{eq:crc}
\end{equation}

guarantees,

\begin{equation}
    \label{eq:crc-eq}
    \E\bigl[L(\hat\lambda;X_{n+1},Y_{n+1})\bigr]\le\alpha
\end{equation}

The goal of \textit{Conformal Adaptation} is to construct a loss $L$ and hierarchical refinement operations, $\mathcal{G}$, satisfying the CRC requirements at partitioning $\coursePart$ while presenting the output at partitioning $\finePart$. Operating on the lower resolution $\coursePart$ reduces computation, and sparsifying $\coursePart$ by selecting candidate cells for refinement with CRC further increases efficiency without compromising fidelity.

\section{Conformal Adaptation}
\label{sec:con-adapt}
\textit{Conformal Adaptation} proposes a trade-off between statistical guarantees and spatial fidelity to speed up computation through sparsification. We are primarily concerned with loss of fidelity due to prediction at the coarser level and seek to close the gap in prediction uncertainty with respect to the finest resolution. Hence, our current focus does not extend the CRC framework to the finest resolution. 
We first establish an abstract formulation to highlight the generalizability of our method, then provide specializations to individual use cases including segmentation and regression.

\begin{figure}[h]
    \centering
\includegraphics[width=0.95\linewidth]{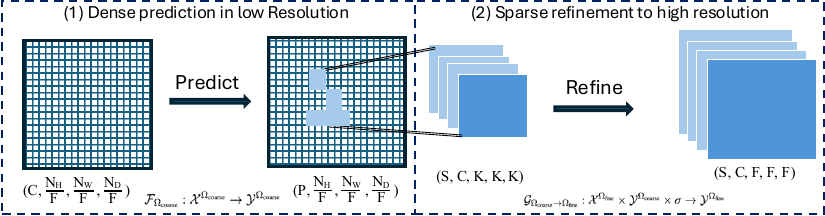}
    \caption{The two-stage process for \textit{Conformal Adaptation} consists of \textbf{(1)} dense predictions at a spatial scale reduced by a factor $r$ and \textbf{(2)} sparse refinement on the selected set $S$ of coarse cells. The selection is based on thresholds calibrated using CRC, and the operating regime we target is $|S|\ll|\coursePart|$.}
    \label{fig:overview}
\end{figure}

Exploiting the spatial coherence of adjacent units in an image-like spatial grid, we reformulate dense prediction as a coarse prediction using an aggregation operator $A$ followed by
\emph{selective} refinement using an operator $\cG$ on relevant regions needing additional computation:

\begin{align}
  \mathcal{F}_{\coursePart} &: \mathcal{X}^{\coursePart} \rightarrow \mathcal{Y}^{\coursePart},\\
  \cG_{\coursePart\rightarrow\finePart} &:
    \cX^{\finePart}\times\cY^{\coursePart}
    \times\{0,1\}^{\coursePart}\to\cY^{\finePart},
    \quad \coursePart\preceq\finePart 
  \label{eq:sparse-G}
\end{align} 

We define a binary selection function $\sigma:\coursePart\rightarrow\{0,1\}$ 
over the coarse partition $\coursePart = \{c_1, \dots, c_K\}$, where a coarse cell $c_k\in\coursePart$ is selected for refinement if and only if $\sigma(c_k)=1$. The resulting set of selected cells is $ S_\sigma=\{c_k\in\coursePart:\sigma(c_k)=1\}\subseteq\coursePart$. $\sigma$ is derived from the coarse prediction $\hat O^{\coursePart}=\mathcal{F}_{\coursePart}(X^{\coursePart})$. Our framework allows for dense pixel-level categorical (binary, multi-class) as well as continuous (regression) predictions. The refinement $\cG_{\coursePart\rightarrow\finePart}\bigl(X^{\finePart},\hat O^{\coursePart},\sigma\bigr)$ acts piecewise on the \emph{fine} sites of each \emph{coarse} cell:

\begin{equation}
    \label{eq:piecewise-G}
    \cG_{\coursePart\rightarrow\finePart}
    \big|_{c_k}
    =\mathcal{G}_{c_k}\left(X^{\finePart}\big|_{c_k},\hat O^{\coursePart}[c_k]\right)
    \qquad c_k \in S_\sigma
\end{equation}

Unselected cells of $\coursePart$ are filled by the null value or simple upsampling, to preserve the prediction structure at fine level. The general inference algorithm is visualized in Fig.~\ref{fig:overview}. The cost for the refinement scales with $|S_\sigma|$ rather than $|\finePart|$, so a smaller selection is cheaper. The task is thus to make $S_\sigma$ as small as possible without discarding signal that matters, a trade-off we control with a statistical guarantee as described in the following paragraphs. The detailed training, calibration, and inference algorithms are provided in Appendix~\ref{appendix:algo}. 
 
\paragraph{Conformal cell selection from the coarse prediction.}
The selection set must be a function of information available at test time, i.e.\ of
the input through the coarse prediction.
We obtain it from a coarse \emph{score field}
$s:\mathcal{X}^{\coursePart}\to[0,1]^{\coursePart}$, read out from $\hat O^{\coursePart}$,
thresholded at a level $\lambda\in[0,1]$:
\begin{align}
    \label{eq:mask}
    \sigma_\lambda(X)[c_k]
  &=\mathbb{I}\left[ s(X^{\coursePart})[c_k]\ge 1-\lambda \right]\\
  S_\lambda(X)&=\{c_k:\sigma_\lambda(X)[c_k]=1\}
\end{align}

Note that the selection is monotone in $\lambda$, with $\lambda_1\le\lambda_2\Rightarrow
S_{\lambda_1}(X)\subseteq S_{\lambda_2}(X)$, and $\lambda=1$ selects every
cell. Crucially, the selection flag, $\sigma_\lambda$, depends on $X$ alone, through the coarse prediction.
 
\paragraph{Cell aggregation from fine labels.}
Since predictions must happen in lower resolution, we must define an \emph{aggregation operator} $A$ that maps the fine labels inside a coarse cell to a \emph{cell weight}. The cell weight scores the `amount' of signal in each coarse cell by a single non-negative scalar. 

\begin{equation}
    w_{c_k}(Y)\;=\;A\!\left(Y^{\finePart}\big|_{c_k}\right)\;\ge\;0,
    \label{eq:weight}
\end{equation}

where $Y^{\finePart}|_{c_k}$ denotes the restriction of the fine-resolution label to the fine sites contained in the coarse cell $c_k$. Because $\coursePart\preceq\finePart$, each coarse cell is a disjoint union of fine sites. Larger weights represent more signal within the coarse cell.


The fine sites of a coarse cell $c\in\coursePart$ are the set
$\finePart|_{c}=\{u\in\finePart: u\subseteq c\}$. Let $\varnothing\in\cY$ denote the null value, corresponding to background labels in segmentation or zero signal in regression.
We require $A$ to satisfy two conditions: 

\begin{condition}[Nullity]\label{cond:c1}
For every $c\in\coursePart$ and every $y\in\cY^{\finePart|_{c}}$, $y\equiv\varnothing \;\Longrightarrow\; A(y)=0$.
\end{condition}

\begin{condition}[Occupancy-preservation]\label{cond:c2}
For every $c\in\coursePart$ and every $y\in\cY^{\finePart|_{c}}$,
$\exists\, u\in\finePart|_{c}$ with $y[u]\neq\varnothing\;\Longrightarrow\; A(y)>0$.
\end{condition}
\paragraph{Connecting to CRC.}
We can now define a loss meeting the requirements of Eq.~\ref{eq:crc}. Let
the loss $L(\lambda; X, Y)$ be the fraction of total signal mass that lies in \emph{unselected}
cells. 
\begin{minipage}{0.53\textwidth}
\begin{equation}
  L(\lambda;X,Y)\;=\;
  \frac{\displaystyle\sum_{c_k\notin S_\lambda(X)} w_{c_k}(Y)}
       {\displaystyle\sum_{c_k\in\coursePart} w_{c_k}(Y)}
  \label{eq:ca-loss}
\end{equation}
\end{minipage}
\begin{minipage}{0.43\textwidth}
\begin{equation}
  \E\!\left[
  \frac{\displaystyle\sum_{c_k\notin S_{\hat\lambda}(X)} w_{c_k}(Y)}
       {\displaystyle\sum_{c_k\in\coursePart} w_{c_k}(Y)}
  \right]\;\le\;\alpha .
  \label{eq:ca-guarantee}
\end{equation}
\end{minipage}

Note that $L$ is bounded in $[0,1]$, as $w \geq 0$, and $0/0:=0$ for examples with zero total ground-truth signal weight (or empty predictions). 
We require $L$ to be non-increasing and right-continuous in $\lambda$. 
Increasing $\lambda$ enlarges $S_\lambda(X)$ by Eq.~\ref{eq:mask}, which removes non-negative terms from the 
numerator.

At $\lambda=1$ every cell is selected, so the numerator is empty 
and $L(1;X,Y)=0\le\alpha$. All conditions for applying CRC are therefore satisfied. 
Substituting Eq.~\ref{eq:ca-loss} into the calibration of Eq.~\ref{eq:crc} yields 
a threshold $\hat\lambda$, which gives us the guarantee in Eq.~\ref{eq:crc-eq} 
reformulated for signal loss in Eq.~\ref{eq:ca-guarantee}.
Under exchangeability with a calibration set, we ensure that the unselected cells contain, in expectation, at most an $\alpha$ fraction of the true signal mass. 
The CRC guarantee is established entirely at $\coursePart$ as it depends on $X$ and $Y$ through the cell aggregates $w_{c_k}(Y)$ and $S_{\hat\lambda}(X)$. We decouple the presentation of the solution from the certification of the solution.

\subsection{Theoretical Guarantees}
\label{sec:guarantees}

We call an aggregation operator $A$ \emph{admissible} if it is fixed and satisfies non-negativity (Eq.~\ref{eq:weight}) and Conditions~\ref{cond:c1} and~\ref{cond:c2}. We write
$\mathcal{O}(Y)=\{c\in\coursePart:\exists\,u\in\finePart|_{c},\;Y[u]\neq\varnothing\}$
for the set of coarse cells \emph{occupied} by the fine label $Y$. These are the coarse cells that have positive signal mass.
$\hat Y^{\finePart}_{\sigma}=\cG_{\coursePart\to\finePart}(X^{\finePart},\hat O^{\coursePart},\sigma)$ is the output of the sparse pipeline under selection $\sigma$, with $\sigma\equiv\mathbf{1}$ recovering dense refinement.

The aforementioned definition of admissible aggregation operators provides sufficient conditions to show that (i) the solution sets are nested in coarse resolution and monotone in $\lambda$, (ii) the missed signal is proportional to the weight and bounded, and (iii) we obtain a marginal guarantee over discarded signal mass. Formally, we have:
\begin{restatable}{theorem}{mytheorem}\label{thm:cert-sparse}
Let $A$ be admissible, let $(X_1,Y_1),\dots,(X_n,Y_n),(X,Y)$ be
exchangeable, where $(X,Y)=(X_{n+1},Y_{n+1})$ is the test pair, and let $\hat\lambda$ be the CRC threshold obtained by
calibrating Eq.~\ref{eq:ca-loss} at level $\alpha$ on the first $n$ pairs. Write
$D=\coursePart\setminus S_{\hat\lambda}(X)$ for the discarded cells. Then:
\begin{enumerate}
\item[(i)] Almost surely,
  $\hat Y^{\finePart}_{\sigma_{\hat\lambda}(X)}$ and $\hat Y^{\finePart}_{\mathbf{1}}$ agree at
  every fine site of every cell in $S_{\hat\lambda}(X)$; the two pipelines can
  differ only on $\bigcup_{c\in D}\finePart|_{c}$.
\item[(ii)] The true signal mass carried by that
  region is controlled:
  \[
    \E\!\left[
      \frac{\sum_{c\in D}w_{c}(Y)}{\sum_{c\in\coursePart}w_{c}(Y)}
    \right]\;\le\;\alpha .
  \]
\item[(iii)] With probability at least
  $1-\alpha/t$ for any $t\in(\alpha,1]$, the discarded region carries at most a
  $t$-fraction of the signal mass; and the event that it carries none is
  exactly $\{\mathcal{O}(Y)\subseteq S_{\hat\lambda}(X)\}$, for every admissible $A$.
\end{enumerate}
\end{restatable}

\begin{proof}[Proof sketch]
    Informally, (i) is a consequence of selectivity of the refinement operation. Coarsening a sparse refinement is the same as masking the coarsening of a dense prediction. The statement (ii) is the result of the admissible operator being compatible with CRC. Finally, (iii) follows from Markov's inequality applied to the non-negative loss. 
    The detailed proof can be found in Appendix~\ref{appendix:math}.
\end{proof}

\section{Performance Metrics}
\label{sec:metrics}
We focus on both dense classification and regression tasks from varying domains. Segmentation tasks use Dice similarity coefficient (DSC) and lesion detection metrics as primary accuracy  targets. The spike count prediction and galaxy super-resolution use mean absolute error (MAE). For performance metrics, we audit both wall-clock speedup as well as FLOPs. We use both CPUs and GPUs to obtain timing results. 

\section{Case Study: Conformal Adaptation for Segmentation}
\label{sec:seg}



\begin{definition}[Segmentation Cell Aggregation] We use fine-grained cell occupation to define our aggregation operator. Each coarse cell $c\in\coursePart$ contains the set of fine sites $\finePart|_c$. We define the aggregation operator for segmentation maps, transporting masks from resolution $\finePart$ to $\coursePart$ as,
\[
w_c(Y^{\finePart})
=
\bigvee_{u\in\finePart|_c} Y^{\finePart}[u],
\qquad c\in\coursePart.
\]
\end{definition}
 
In the case of tumor segmentation, the logical OR operator ensures positive tumor cells are not missed during coarsening. By construction, we can see that the OR-aggregation operator is zero-preserving, as it is 0 when all inputs are 0, and monotonically increasing when any cell is nonzero. We take any off-the-shelf segmentation prediction model providing softmax-like scores per cell, and we can define the selection set $S_\lambda(X)$.

\subsection{Brain Tumor Segmentation}

The brain tumor dataset from the Medical Segmentation Decathlon (MSD)~\citep{MedDecathlon} consists of 484 cases with brain tumors of varying sizes, occupying at most 3.5\% of the image volume. Each case is a 4D volume combining pre- and post-contrast T1-weighted, T2-weighted, and fluid attenuation inverse recovery (FLAIR) MRI scans. 

We apply \textit{Conformal Adaptation} to three segmentation architectures: U-Net~\citep{unet}, Swin-UNETR~\citep{swin-unet, swinunetr-v2}, and VISTA3D~\citep{vista3d}. We train each randomly initialized architecture three times at full resolution (\textit{Full}) and at \(4\times\)-downsampled resolution (\textit{Coarse}). The coarse model then determines the sparse patches used to train the second-stage refiner (\textit{Coarse + Refine}). The data processing, architecture hyperparameters, and training methods are detailed in Appendix~\ref{appendix:decathlon}. Across architectures and devices, the method provides substantial FLOPs and wall-clock reductions while retaining similar DSC (Fig.~\ref{fig:seg-accuracy-speed}). In CPU-only inference, VISTA3D achieves up to $72\times$ wall-clock speedup and an $82.1\times$ reduction in FLOPs (Table~\ref{tab:model_comparison}). We see super-linear improvement as the data can use fewer forward passes with sliding window inference in coarse resolution.

\begin{table}[t] \centering \caption{Performance comparison on MSD Task01. Values are reported as mean $\pm$ standard deviation over three seeds. Columns 2-4 are the DSC for each model arrangement. Capacity-constrained models like U-Net see an improvement in performance due to additional weights. Larger scans require sliding window inference, increasing the FLOPs required for a single sample. } \label{tab:model_comparison} 
\resizebox{\textwidth}{!}{
\begin{tabular}{lccccc} \toprule 
Model & Full & Coarse & Coarse + Refine ($\alpha=0.1$) & Cell Miss Rate & FLOPs Speedup \\ \midrule 
U-Net & $0.63 \pm 0.03$ & $0.51 \pm 0.00$ & $0.68 \pm 0.01$ & $0.10 \pm 0.01$  & 2.70$\times$ \\
Swin-UNETR & $0.74 \pm 0.01$ & $0.61 \pm 0.01$ & $0.70 \pm 0.01$ & $0.09 \pm 0.00$  & 54.03$\times$ \\
VISTA3D & $0.74 \pm 0.00$ & $0.62 \pm 0.01$ & $0.71 \pm 0.00$ & $0.08 \pm 0.01$  & 82.10$\times$ \\ \bottomrule
\end{tabular}}
\end{table}

\begin{figure}[h]
    \centering    \includegraphics[width=0.9\linewidth]{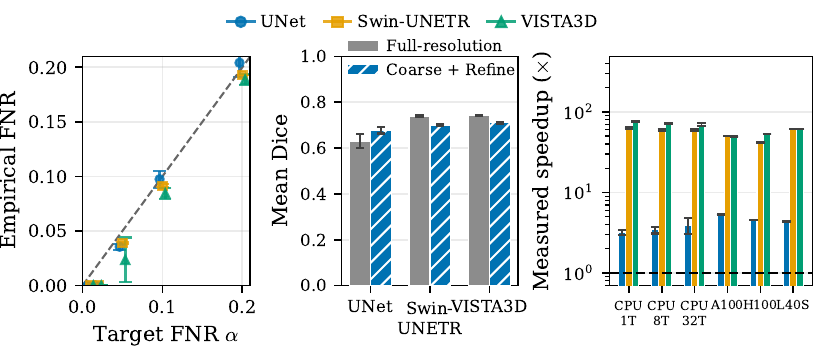}
    \caption{\textbf{Left:} Empirical verification of the CRC calibration quality. \textbf{Middle:} mean DSC of the Full baseline against the Coarse + Refine model at $\alpha=0.1$. \textbf{Right:}  Corresponding measured wall-clock speedup over the Full baseline on different devices.}
    \label{fig:seg-accuracy-speed}
\end{figure}
\vspace{-1em}
\begin{figure}[h]
    \centering    \includegraphics[width=0.6\linewidth]{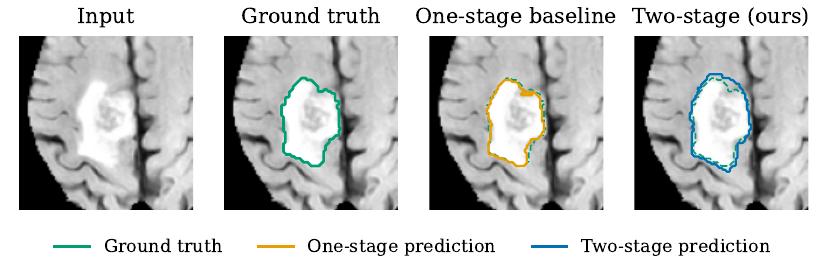}
    \caption{Predicted tumor boundaries on a held-out test scan (Swin-UNETR, $\alpha=0.1$). The Coarse + Refine model tracks the boundary as closely
    as the Full baseline at a fraction of the FLOPs (Table~\ref{tab:model_comparison}).}
    \label{fig:seg-contour-grid}
\end{figure}

\subsection{Head and Neck Cancer and Metastatic Lymph Node Segmentation}
Accurate segmentation of primary head and neck cancer (HNC) and affected lymph node volumes is essential for radiotherapy planning (RTP) as well as for monitoring tumor response to treatments. Computed tomography (CT) and 2-[18F]fluoro-2-deoxyglucose (FDG) PET are imaging modalities commonly used to localize and delineate tumors for RTP.
We base our study on the multimodal HECKTOR 2026 grand challenge dataset~\citep{HECKTOR26}, which provides 782 FDG-PET and CT volumes spanning from the head to the pelvis. The size, structure, and sparsity of the HNC scans and tumors make the dataset an ideal candidate for \textit{Conformal Adaptation}, as seen in Fig.~\ref{fig:hn-data-characteristics}. Each tumor spans significantly less than the total scanned area at the standard \(1\times1\times1\,\mathrm{mm}^3\) spacing. 
We train and evaluate off-the-shelf models, including nnU-Net~\citep{nnunet}, an industry standard that establishes a strong baseline, and the transformer-CNN-based Swin-UNETR~\citep{swinunetr-v2}. We use the standard CT and PET data pre-processing techniques, such as normalization to \([0,1]\) and isotropic resampling. During training, we achieve a large speedup, as training the coarse model along with a sparse refiner takes $15\times$ fewer FLOPs than the full-resolution model (see Fig.~\ref{fig:hn-data-characteristics}).  Full data-processing and training details are presented in Appendix~\ref{appendix:hn}.

\begin{figure}[t]
    \centering
    \begin{subfigure}[b]{0.31\textwidth}       \includegraphics[width=\textwidth]{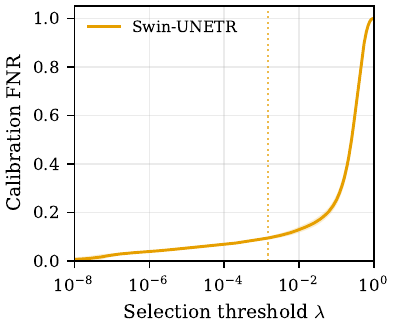}
    \end{subfigure}
    \begin{subfigure}[b]{0.31\textwidth}       \includegraphics[width=\textwidth]{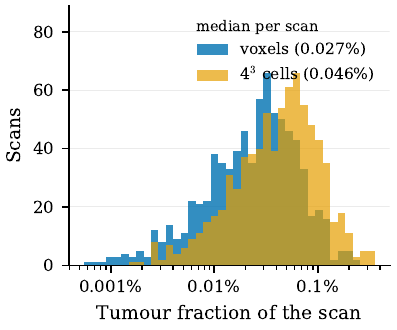}
    \end{subfigure}
    \begin{subfigure}[b]{0.31\textwidth}
\includegraphics[width=\textwidth]{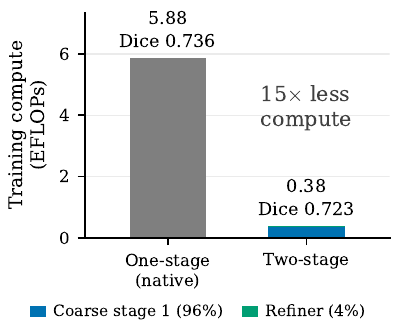}
    \end{subfigure}
    \caption{\textbf{(Left)}: False negative rates (FNR) can rapidly increase depending on chosen threshold, requiring a principled methods like CRC to choose appropriate  $\lambda$ (See Appendix~\ref{appendix:threshold}.\textbf{(Middle)} Tumor cells span a small fraction of the total scanned volume in most scans in the dataset. \textbf{(Right)} Training the coarse model on the aggregated samples and a separate refiner to convergence requires $15\times$ fewer FLOPs. All models are trained until validation loss stops decreasing with a set patience level.}
    \label{fig:hn-data-characteristics}
\end{figure}

Once trained and calibrated, our Coarse + Refine model achieves results comparable to the Full model (DSC 0.72 vs.\ 0.74) while requiring $68\times$ fewer FLOPs (Fig.~\ref{fig:hn-results}). Test-time augmentation with mirroring is used on all dense models. Furthermore, for lesion-detection metrics, where the predicted volume is used to detect lesions, our method achieves comparable recall and precision metrics but suffers from missing small adversarial GTVn cases (see Table~\ref{tab:hn-lesion-results}). 

\begin{figure}[h]
    \centering
    \begin{subfigure}[b]{0.66\textwidth}       \includegraphics[width=\textwidth]{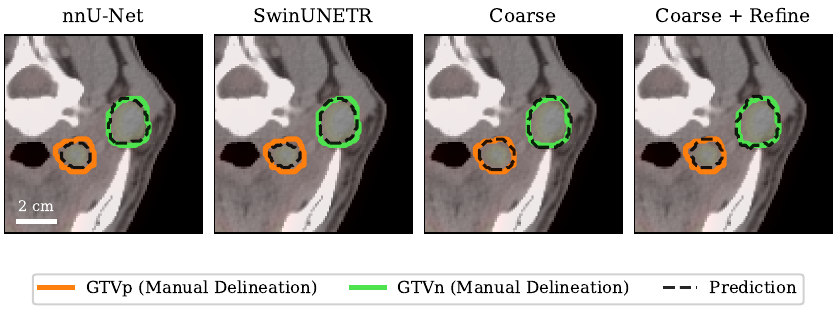}
    \end{subfigure}
    \begin{subfigure}[b]{0.32\textwidth}
\includegraphics[width=\textwidth]{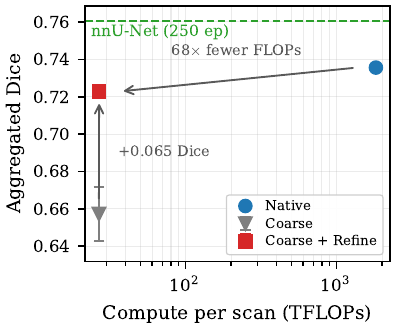}
    \end{subfigure}
    \caption{(Left) Visualizations of model segmentation with ground truth showing Coarse + Refine models with comparable contours. (Right) A coarse inference reduces per-scan compute cost while a sparse refinement improves prediction quality with little additional expense. }
    \label{fig:hn-results}
\end{figure}
\section{Case Study: Conformal Adaptation for Regression}
\label{sec:reg}

We evaluate the generality of the cell weight mechanism defined in Eq.~\ref{eq:weight} for regression-type problems, where the prediction target is a continuous, non-negative signal (a rate, a count, a flux) rather than a discrete class label.

\begin{definition}[Regression Cell Aggregation]
For a continuous, non-negative fine-resolution target $Y\in\mathcal{Y}^{\finePart}$, we define the aggregation operator as the total signal mass inside the coarse cell,
$
    w_c(Y^{\finePart}) = \sum_{u\in\finePart|_c} Y^{\finePart}[u], \quad c\in\coursePart .
$
\end{definition}
Summation is zero-preserving  and, since $Y\ge 0$, monotonically non-decreasing in any added signal, satisfying the requirements of Eq.~\ref{eq:weight}. The CRC loss of Eq.~\ref{eq:ca-loss} then bounds, in expectation, the \emph{fraction of total signal mass} that falls in unselected regions. Regression scores (e.g.,\ a predicted upper quantile of cell mass) are not confined to $[0,1]$, so we apply Eq.~\ref{eq:mask} after a fixed, strictly increasing map into $[0,1]$. The map preserves score order, so the selected sets stay nested in $\lambda$ as CRC requires, and we report $\hat\lambda$ by its equivalent raw-score cutoff. 

\subsection{CASCADE: 1D Calcium-Imaging Spike-Rate Inference}
\label{sec:cascade-case-study}

We use the CASCADE benchmark~\citep{rupprecht2021database} with 515 samples across 117 sources,
in which the task is to convert fluorescence traces from calcium imaging into spike-rate curves. 
The coarse model $\mathcal{F}_{\coursePart}$ block-averages the fluorescence trace by a factor of eight before applying the 1D U-Net of the full-resolution baseline $\mathcal{F}_{\finePart}$. Note, we do not use the benchmark's decoder-free sliding-window regressor. 
CRC is calibrated on the coarse block counts, and the certified selection is refined with a count-predicting local refiner combined with a mass-preserving Gaussian reconstruction kernel, so bounding missed count mass at the coarse resolution also bounds missed mass in the reconstructed rate curve. An example input and output prediction is provided in Fig.~\ref{fig:cascade-predictions}.
Table~\ref{tab:cascade-headline} reports the reference operating point. 
Across $\alpha\in\{0.01,0.05,0.1,0.2\}$, CRC calibration is empirically valid (test risk $\le\alpha$ at every level), and the reconstructed-curve MAE stays 
within $0.05$--$0.10$\,Hz of the full-resolution baseline (median MAE $0.250$\,Hz) while the retained fraction of coarse cells falls to $8.4\%$ at $\alpha{=}0.1$, an analytic FLOPs speedup of $4.77\times$. All details of the data pre-processing, model architecture, and calibration can be found in Appendix~\ref{appendix:CASCADE}.

\begin{figure}[t]
    \centering
\includegraphics[width=0.85\linewidth]{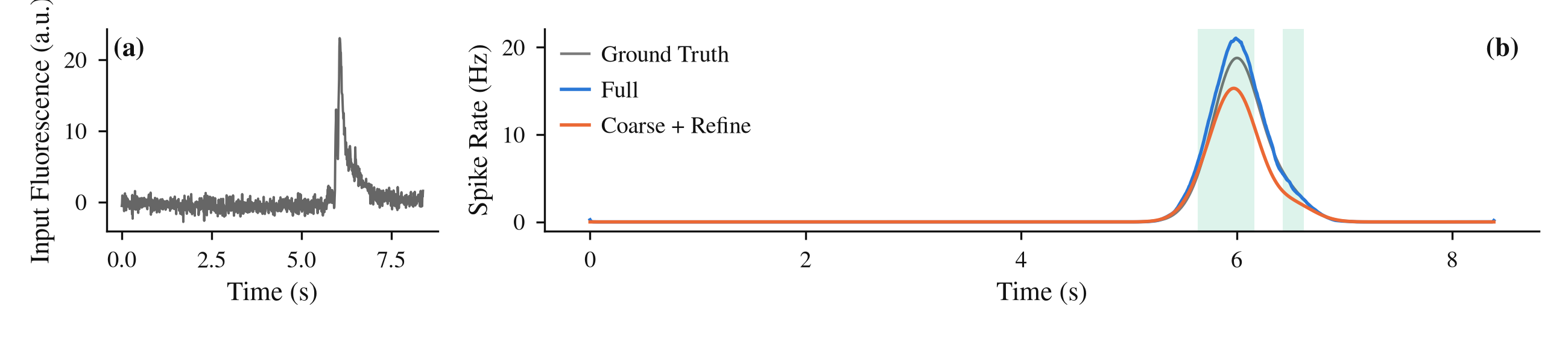}
    \caption{(a) Input signal. (b) Example test window: ground-truth spike rate overlaid with the Full prediction and the Coarse + Refine reconstruction, at a retained fraction of $8.4\%$.}
    \label{fig:cascade-predictions}
\end{figure}

\subsection{Galaxy Flux Field Super-Resolution}
\label{sec:galsim-case-study}
Extending our experiments further, we also study a 2D galaxy-flux field task of predicting galaxy shape and flux from simulated surface-brightness maps. Simulated surface-brightness fields are foundational to observational astronomy. As
the raw data provide no labeled truth, surveys characterize what their pipelines detect and miss by injecting sources of known flux into images and measuring recovery
\citep{anbajagane2025balrog, vera2026dp1}. The resulting completeness and photometric bias curves are used as inputs for correction terms on downstream detectors. Fields are generated with \textsc{GalSim} \citep{rowe2015galsim}, the community-standard image simulator validated against the precision requirements of Stage~IV surveys. In practice surveys such as Rubin/LSST process on the order of $20$\,TB per night under a latency budget of seconds per visit \citep{ivezic2019lsst} over fields that are overwhelmingly empty sky, making this setting well suited to our method.

We use a 2D U-Net as our base model. The Coarse + Refine pipeline is a coarse model sharing the base model's architecture but applied to an $8\times$-downsampled field, CRC certifying coarse cells by flux mass, and a sparse refiner predicting a non-negative spatial spread within each certified cell. The normalized refiner output is multiplied by the coarse net's own predicted median flux, making flux conservation between coarse and refined predictions exact by construction. 

At the calibrated operating point ($\alpha{=}0.1$, score cutoff $7.69$), 
the pipeline retains only $2.2\%$ of coarse cells on average, for an analytic FLOPs speedup of $61.6\times$ over the full-resolution baseline, at the cost of a higher per-pixel MAE ($0.134$ vs.\ $0.046$ for the full-resolution baseline) with total-flux relative error of $5.6\%$ vs.\ $2.3\%$, consistent with the certification. 
Figure~\ref{fig:galsim-main} presents example predictions and measured wall-clock speedups across several devices, illustrating the method's potential for edge deployment.

\begin{figure}[h]
    \centering
    \includegraphics[width=\linewidth]{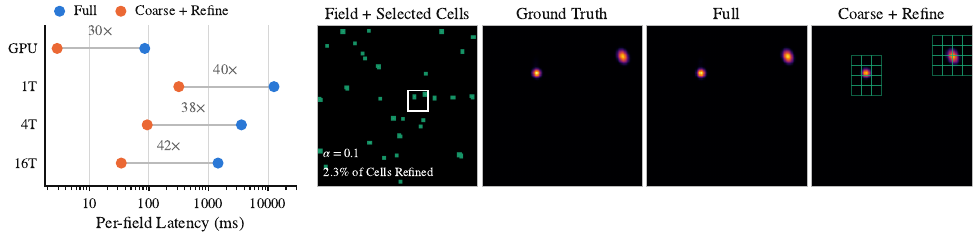}
    \caption{(Left) Achieved real-world wall-clock speedups on GPUs and CPU thread counts. (Right) Sample selection and comparison between the Full and Coarse + Refine models. The CRC-certified Coarse + Refine predictions have little discrepancy while costing a fraction of the compute and time. }
    \label{fig:galsim-main}
\end{figure}

\section{Related Work}
\label{sec:related-work}

Static compression reduces inference cost by modifying the \emph{model} independently of its input. Common approaches include pruning~\citep{yu_compressing_2017, osawa_accelerating_2017}, architectural simplification~\citep{slimUNETR, slimUNETRV2}, static activation sparsification~\citep{alanova_motivating_2026}, and quantization~\citep{gupta_deep_2015, sun_ultra-low_2020, liu_llm-qat:_2024, chen_efficientqat:_2025}. \citet{tokenseg} propose token sparsification for efficient 3D segmentation, while static resolution scaling accelerates training by operating at fixed lower resolutions~\citep{fastdinov2}. Unlike our method, these approaches evaluate error post hoc and provide no accuracy guarantees. Their gains are largely orthogonal to our input-adaptive sparsification.

Input-adaptive methods allocate compute based on the sample. Examples include token sparsification in vision transformers~\citep{rao_dynamicvit:_2021, fayyaz_adaptive_2022, uddin_heart-vit:_2025}, region-wise token granularity selection~\citep{havtorn2023msvit}, early exiting~\citep{he_magic_2021}, and adaptive pixel refinement in PointRend~\citep{pointrend}. While these methods focus computation on informative regions, they rely on learned or heuristic selection mechanisms without guarantees on discarded information, making them susceptible to failure modes and adversarial attacks~\citep{goyal_survey_2026}. In contrast, our calibrated selection rule provably bounds discarded signal mass at risk level $\alpha$.
Conformal prediction has been extended to structured outputs through unified task formulations~\citep{csp}, segmentation-specific spatial modeling~\citep{consign, CRC}, and input-conditional coverage~\citep{luo2025conditional}. These methods treat the predictor as fixed and use conformal prediction only to construct uncertainty sets, while we use the conformal certificate to guide computation.  \citet{wang2026conformal}, \citet{schuster2022confident}, \citet{schuster_consistent_2021}, and \citet{jazbec_fast_2024} apply distribution-free guarantees for adaptive inference and calibrate early-exit strategies to match full-model performance under compute constraints.

\section{Discussion, Limitations, and Conclusion}
\label{sec:limitations}
Although we evaluate on uniform grids, the derivation in Sec.~\ref{sec:con-adapt} is domain-agnostic and extends to non-uniform domains (e.g., graphs) whenever the refinement operators preserve the hierarchical partition. 
We use identical backbones for the native and coarse passes to isolate the gains due to conformalized selection itself. 
\textit{Conformal Adaptation} decides \emph{where} to compute, whereas quantization, pruning, sparse architectures, and token selection reduce the cost of \emph{each} computation, so the two are orthogonal. 
Any stage may use quantized or pruned backbones~\citep{liu_llm-qat:_2024, yu_compressing_2017}, sparse architectures~\citep{slimUNETR, slimUNETRV2, tokenseg}, or token selection~\citep{rao_dynamicvit:_2021, uddin_heart-vit:_2025} without changing the calibration procedure or the guarantee; whether their savings compound with ours is promising future work.
We acknowledge that the guarantee has three limitations. First, it bounds the expected fraction of true signal in discarded cells, not the error of the refined prediction. Second, it holds marginally, so an individual scan may exceed $\alpha$, and small lesions that are faint at coarse resolution may be missed. Third, it assumes exchangeability between calibration and test data and is therefore sensitive to covariate shift, e.g., across scanners or sites. Weighted CRC~\cite{tibshirani2019conformal}, high-probability (PAC-style) bounds~\citep{ltt}, and per-lesion risk control are natural extensions. \textit{Conformal Adaptation} offers a simple, calibrated way to spend compute only where it is required.

\newpage
\subsection*{Aknowledgements}
Research reported was supported by NCI of the National Institutes of Health under grant number R01CA258821, P30CA008748, and R01EB032825. 
The content is solely the responsibility of the authors and does not necessarily represent the official views of the National Institutes of Health.
This work utilized resources from the High Performance Computing Group at Memorial Sloan Kettering Cancer Center.

\subsection*{AI use statement}

In this work, we used generative AI tools to help refine and audit mathematical claims, assist in writing proofs, and assist in writing code to implement methods.
We have not used generative AI tools for generating synthetic datasets, help develop theoretical models or conceptual frameworks, propose or refine hypotheses, design or provide feedback on research methodology and experiments, assist with translation, clean and reformat datasets, support qualitative and thematic data analysis, and interpret results. 

Additionally, we used generative AI tools to edit code, improve readability, check grammar, and help identify relevant literature. 
We have reviewed all AI-assisted work. Author 1 has reviewed the implementation assisted with AI as well as the mathematical claims that were refined using AI. Authors 2 and 3 have further verified the notation and edits that were aided by generative AI, including checking mathematical claims, references, and manuscript edits. 
We take responsibility for the final content of this work,
including text, claims or artifacts produced with the aid of generative AI.

\subsection*{Reproducibility statement}

We take multiple steps to ensure all our results and analysis are reproducible. 
Our analysis uses publicly available data that can be freely downloaded from the cited sources. Although open to the public, the HECKTOR 2026 dataset requires signing a data use agreement and cannot be shared under its license terms. We detail the required data pre-processing in Appendices~\ref{appendix:decathlon}, \ref{appendix:hn}, \ref{appendix:CASCADE}, and~\ref{appendix:galsim}. The model definitions, training algorithms, and post-processing are also detailed in these appendices. Our main algorithmic contribution is detailed in Appendix~\ref{appendix:algo}. To aid in reviewing mathematical correctness, the full proof is given in Appendix~\ref{appendix:math}. Finally, we provide all anonymized pre- and post-processing, training, and visualization code, along with all seeds used, as supplementary material.

\bibliography{biblio}
\bibliographystyle{iclr2027/iclr2027_conference}

\newpage
\appendix
\section{Conformal Adaptation: Training, Calibration, Inference}
\label{appendix:algo}
\subsection{Training}
Given an admissible operator $A$, a coarse predictor model $\mathcal{F}_{\coursePart}$ is trained together with the refinement operator $\mathcal{G}$. In general, training the two models is a three-step process: 

\begin{enumerate}
    \item \textbf{Train the coarse model}: 
    \begin{itemize}
        \item Downsample the input to the coarse partitioning, and aggregate the labels into coarse cell weights with an admissible operator $A$
    \end{itemize}
    \item \textbf{Extract sparse training data}:
    \begin{itemize}
        \item Rather than the complete fine-grained samples, the sparse refiner is trained on small patches (that will be selected by the coarse model)
        \item The sparse models can be trained either with `simulated' sparse data that is selected by an oracle or  it can be selected using the now trained coarse model $\mathcal{F}_{\coursePart}$ at a training selection level $\lambda_{\text{train}}$. Setting $\lambda_{\text{train}}=\hat\lambda$ (Algorithm~\ref{alg:calibration}) trains the refiner on the certified distribution it will see at inference, at the cost of one refiner per calibrated $\alpha$
    \end{itemize}
    \item \textbf{Train the refiner model}
    \begin{itemize}
        \item The refiner model is trained on the extracted sparse patch-level data.
    \end{itemize}
\end{enumerate}

Throughout this section we use $\mathcal{G}$ for $\mathcal{G}_{\coursePart\rightarrow\finePart}$ for conciseness. 

\begin{algorithm}[h]
\caption{Training the Coarse and Selective Refinement Models}
\label{alg:training}
\begin{algorithmic}[1]
\Require Training dataset $\mathcal{D}_{\text{train}} = \{(X_j^{\finePart}, Y_j^{\finePart})\}_{j=1}^{N}$, downsampling operator $\operatorname{Downsample}(\cdot)$, admissible aggregation operator $A$, scoring function $s(\cdot)$, training selection level $\lambda_{\text{train}} \in [0,1]$.
\Ensure Trained coarse model $\mathcal{F}_{\coursePart}$, trained refinement model $\mathcal{G}$.

\Statex \textbf{// Step 1: Coarse Data Preparation \& Training}
\State Construct downsampled coarse dataset:
\Statex \quad $\mathcal{D}_{\coursePart} \gets \bigl\{ \bigl( \operatorname{Downsample}(X_j^{\finePart}),\, w(Y_j^{\finePart}) \bigr) \bigr\}_{j=1}^N$, where $w(Y)=(w_{c}(Y))_{c\in\coursePart}$
\State Train $\mathcal{F}_{\coursePart} \colon \cX^{\coursePart} \to \cY^{\coursePart}$ by minimizing standard empirical risk on $\mathcal{D}_{\coursePart}$

\Statex \textbf{// Step 2: Refinement Data Extraction via Coarse Selection}
\State Initialize refinement dataset $\mathcal{D}_{\text{refine}} \gets \emptyset$
\For{each sample $(X_j^{\finePart}, Y_j^{\finePart}) \in \mathcal{D}_{\text{train}}$}
    \State Compute coarse prediction: $\hat{O}_j^{\coursePart} \gets \mathcal{F}_{\coursePart}(X_j^{\coursePart})$
    \For{each coarse cell $c_k \in \coursePart$}
        \State Select cell by Eq.~\ref{eq:mask}: $\sigma_j[c_k] \gets \mathbb{I}[s(X_j^{\coursePart})[c_k] \ge 1-\lambda_{\text{train}}]$
        \Statex \qquad\quad \textit{(oracle variant: $\sigma_j[c_k] \gets \mathbb{I}[w_{c_k}(Y_j) > 0]$)}
        \If{$\sigma_j[c_k] = 1$}
            \State Extract fine patch $x_{j,k}^{\finePart} \gets X_j^{\finePart}\big|_{c_k}$
            \State Extract ground-truth fine label $y_{j,k}^{\finePart} \gets Y_j^{\finePart}\big|_{c_k}$
            \State $\mathcal{D}_{\text{refine}} \gets \mathcal{D}_{\text{refine}} \cup \left\{ \left( (x_{j,k}^{\finePart}, \hat{O}_j^{\coursePart}[c_k]),\, y_{j,k}^{\finePart} \right) \right\}$
        \EndIf
    \EndFor
\EndFor

\Statex \textbf{// Step 3: Train Refinement Model}
\State Train local refinement operator $ \mathcal{G}\colon \cX^{\finePart}\big|_{c_k} \times \cY^{\coursePart}[c_k] \to \cY^{\finePart}\big|_{c_k}$ on $\mathcal{D}_{\text{refine}}$
\State \Return $\mathcal{F}_{\coursePart}, \mathcal{G}$
\end{algorithmic}
\end{algorithm}

\subsection{Calibration}
The CRC algorithm is run on the low-resolution partitioning as seen in Algorithm~\ref{alg:calibration}. It uses the missed-signal-mass loss of Eq.~\ref{eq:ca-loss}, which depends only on the coarse selection and the cell weights, so the refiner $\mathcal{G}$ plays no part in calibration.
\begin{algorithm}[h]
\caption{Conformal Risk Calibration (Threshold Selection)}
\label{alg:calibration}
\begin{algorithmic}[1]
\Require Calibration set $\mathcal{D}_{\text{cal}} = \{(X_i^{\finePart}, Y_i^{\finePart})\}_{i=1}^{n}$, trained coarse model $\mathcal{F}_{\coursePart}$, downsampling operator $\operatorname{Downsample}(\cdot)$, admissible aggregation operator $A$, scoring function $s(\cdot)$, risk level $\alpha \in (0, 1)$, grid of candidates $\Lambda \subseteq [0,1]$ with $1 \in \Lambda$.
\Ensure Calibrated selection threshold $\hat{\lambda}$.

\For{$i = 1$ \textbf{to} $n$}
    \State $X_i^{\coursePart} \gets \operatorname{Downsample}(X_i^{\finePart})$
    \State $\hat{O}_i^{\coursePart} \gets \mathcal{F}_{\coursePart}(X_i^{\coursePart})$
    \State Cell weights (Eq.~\ref{eq:weight}): $w_{c_k}(Y_i) \gets A\bigl(Y_i^{\finePart}\big|_{c_k}\bigr)$ for each $c_k \in \coursePart$
\EndFor
\For{each candidate threshold $\lambda \in \Lambda$}
    \For{$i = 1$ \textbf{to} $n$}
        \State Form selection vector $\sigma_\lambda \in \{0, 1\}^{\coursePart}$:
        \Statex \qquad\quad $\sigma_\lambda[c_k] \gets \mathbb{I}[s(X_i^{\coursePart})[c_k] \ge 1-\lambda]$
        \State Missed signal mass (Eq.~\ref{eq:ca-loss}, $0/0 := 0$):
        \Statex \qquad\quad $L_i(\lambda) \gets \sum_{k:\,\sigma_\lambda[c_k]=0} w_{c_k}(Y_i) \,\big/\, \sum_{k} w_{c_k}(Y_i)$
    \EndFor
    \State Compute empirical risk: $\widehat{R}_n(\lambda) \gets \frac{1}{n} \sum_{i=1}^n L_i(\lambda)$
\EndFor
\State Select threshold guaranteeing controlled risk ($L \in [0,1]$, so $B=1$):
\Statex \quad $\hat{\lambda} \gets \inf \left\{ \lambda \in \Lambda : \tfrac{1}{n+1}\bigl(n\,\widehat{R}_n(\lambda) + 1\bigr) \le \alpha \right\}$
\State \Return $\hat{\lambda}$
\end{algorithmic}
\end{algorithm}

\subsection{Inference}
Once trained and calibrated, the coarse and refiner models $\mathcal{F}_{\coursePart}$ and $\mathcal{G}$ can be used for inference along with the calibrated threshold $\hat{\lambda}$. The inference algorithm is detailed in Algorithm~\ref{alg:inference}.

\begin{algorithm}[h]
\caption{Conformal Adaptation Inference}
\label{alg:inference}
\begin{algorithmic}[1]
\Require Test input $X^{\finePart}$, models $\mathcal{F}_{\coursePart}, \mathcal{G}$, calibrated threshold $\hat{\lambda}$, downsampling operator $\operatorname{Downsample}(\cdot)$, scoring function $s(\cdot)$.
\Ensure Sparsely refined prediction $\hat{Y}^{\finePart}$.

\State Downsample coarse input: $X^{\coursePart} \gets \operatorname{Downsample}(X^{\finePart})$
\State Run coarse model: $\hat{O}^{\coursePart} \gets \mathcal{F}_{\coursePart}(X^{\coursePart})$
\State Initialize fine canvas: $\hat{Y}^{\finePart} \gets \operatorname{Upsample}(\hat{O}^{\coursePart})$
\State Determine sparse refinement set: 
\Statex \quad $S \gets \left\{ c_k \in \coursePart : s(X^{\coursePart})[c_k] \ge 1-\hat{\lambda} \right\}$

\For{each coarse cell $c_k \in S$} \Comment{Execute refinement strictly on sparse locations}
    \State $\hat{Y}^{\finePart}\big|_{c_k} \gets \mathcal{G}\left(X^{\finePart}\big|_{c_k}, \hat{O}^{\coursePart}[c_k]\right)$
\EndFor

\State \Return $\hat{Y}^{\finePart}$
\end{algorithmic}
\end{algorithm}
\newpage
\section{Derivation of Guarantees}
\label{appendix:math}
\begin{lemma}\label{lem:crc-hyp}
For every $(X,Y)$, the loss of Eq.~\ref{eq:ca-loss} satisfies
$L(\lambda;X,Y)\in[0,1]$, is non-increasing and right-continuous in
$\lambda\in[0,1]$, and satisfies $L(1;X,Y)=0$.
\end{lemma}

\begin{proof}
Non-negativity of $A$ makes every $w_{c}(Y)\ge 0$; the numerator sums a
subset of the terms of the denominator, so $L\in[0,1]$ (with $0/0:=0$).
The denominator does not depend on $\lambda$. If $\lambda_1\le\lambda_2$ then
$S_{\lambda_1}(X)\subseteq S_{\lambda_2}(X)$ by Eq.~\ref{eq:mask}, so passing
from $\lambda_1$ to $\lambda_2$ deletes non-negative terms from the numerator
and $L(\lambda_2;X,Y)\le L(\lambda_1;X,Y)$. Since $\coursePart$ is finite, the score
field $s$ takes finitely many values, and the closed threshold in
Eq.~\ref{eq:mask} makes $\lambda\mapsto S_\lambda(X)$ right-continuous; hence so
is $L$. At $\lambda=1$ every cell satisfies $s(X^{\coursePart})[c]\ge 0$, so
$S_1(X)=\coursePart$, the numerator is an empty sum, and $L(1;X,Y)=0$.
\end{proof}

\begin{lemma}\label{lem:commute-seg}
Under Conditions~\ref{cond:c1} and~\ref{cond:c2}, for every $Y\in\cY^{\finePart}$
and every $c\in\coursePart$,
\[
  w_{c}(Y)>0\iff c\in\mathcal{O}(Y).
\]
Consequently, for every $\lambda\in[0,1]$,
\[
  L(\lambda;X,Y)=0
  \iff
  \mathcal{O}(Y)\subseteq S_\lambda(X),
\]
and the right-hand event depends on neither $A$ nor $\lambda$ except through
$S_\lambda(X)$. In particular all admissible operators induce losses with the
same zero set.
\end{lemma}

\begin{proof}
If $c\notin\mathcal{O}(Y)$ then $Y^{\finePart}|_{c}\equiv\varnothing$ and
Condition~\ref{cond:c1} gives $w_c(Y)=0$. If $c\in\mathcal{O}(Y)$ then some
fine site of $c$ carries a non-null label and Condition~\ref{cond:c2} gives
$w_c(Y)>0$. This proves the equivalence.

For the second claim, suppose first that $\mathcal{O}(Y)=\emptyset$. Then all
weights vanish, the convention $0/0:=0$ gives $L=0$, and the inclusion holds
trivially. Otherwise the denominator of Eq.~\ref{eq:ca-loss} is strictly
positive, so $L(\lambda;X,Y)=0$ iff $\sum_{c\notin S_\lambda(X)}w_c(Y)=0$,
iff $w_c(Y)=0$ for every $c\notin S_\lambda(X)$ (all terms being
non-negative), iff no occupied cell lies outside $S_\lambda(X)$ by the
equivalence just proved.
\end{proof}

\begin{lemma}\label{lem:commute-refine}
Let $\sigma\in\{0,1\}^{\coursePart}$ with selected set $S$, and write
$\hat Y^{\finePart}_{\sigma}:=\cG_{\coursePart\to\finePart}(X^{\finePart},\hat O^{\coursePart},\sigma)$.
Then
\[
  \hat Y^{\finePart}_{\sigma}\big|_{c}=\hat Y^{\finePart}_{\mathbf{1}}\big|_{c}
  \quad\text{for every } c\in S,
  \qquad
  \hat Y^{\finePart}_{\sigma}\big|_{c}=\hat Y^{\finePart}_{\mathbf{0}}\big|_{c}
  \quad\text{for every } c\notin S,
\]
where $\hat Y^{\finePart}_{\mathbf 0}=\operatorname{Upsample}(\hat O^{\coursePart})$.
Consequently, for any cellwise functional $\phi$ (i.e.\ $\phi(Y)[c]$ depends
only on $Y|_c$),
\[
  \phi\bigl(\hat Y^{\finePart}_{\sigma}\bigr)[c]
  = \sigma[c]\,\phi\bigl(\hat Y^{\finePart}_{\mathbf 1}\bigr)[c]
  + \bigl(1-\sigma[c]\bigr)\,\phi\bigl(\hat Y^{\finePart}_{\mathbf 0}\bigr)[c],
\]
and in particular
\[
  \mathcal{O}\bigl(\hat Y^{\finePart}_{\sigma}\bigr)
  = \bigl(S\cap\mathcal{O}(\hat Y^{\finePart}_{\mathbf{1}})\bigr)
  \,\cup\,
  \bigl(\mathcal{O}(\hat Y^{\finePart}_{\mathbf{0}})\setminus S\bigr).
\]
Sparse refinement is a cellwise selection between two fields, neither of which
depends on $\sigma$.
\end{lemma}

\begin{proof}
By Eq.~\ref{eq:piecewise-G}, the output of $\cG_{\coursePart\to\finePart}$ on a
selected cell $c$ is $\cG_{c}(X^{\finePart}|_c,\hat O^{\coursePart}[c])$. Its
arguments do not depend on $\sigma$, so it coincides with the value on $c$ when
$\sigma=\mathbf 1$. On an unselected cell, $c$ is never overwritten, so the
output retains the initial canvas
$\operatorname{Upsample}(\hat O^{\coursePart})|_c$
(Algorithm~\ref{alg:inference}, line~3). This canvas is computed from
$\hat O^{\coursePart}$ alone, so it is independent of $\sigma$ and equals
$\hat Y^{\finePart}_{\mathbf 0}|_c$. The formula for a cellwise $\phi$ follows by
evaluating $\phi$ on each cell in the two cases. The occupancy identity is the
instance $\phi=\mathbb{I}[c\in\mathcal O(\cdot)]$.
\end{proof}

\begin{remark}
If $\operatorname{Upsample}$ maps every cell to $\varnothing$ (e.g.\ a
background fill in segmentation), then
$\mathcal O(\hat Y^{\finePart}_{\mathbf 0})=\emptyset$ and the lemma reduces to
$\mathcal O(\hat Y^{\finePart}_\sigma)=S\cap\mathcal O(\hat Y^{\finePart}_{\mathbf 1})$.
For regression, $\operatorname{Upsample}$ may be any fixed interpolant
(nearest, bilinear, etc.). It need not be cell-local; it only needs to be
computed from $\hat O^{\coursePart}$ independently of $S$.
\end{remark}

\mytheorem*

\begin{proof}
For (ii), we see that Lemma~\ref{lem:crc-hyp} verifies the hypotheses of the CRC guarantee
(Eq.~\ref{eq:crc}) for the family $\{L(\lambda;\cdot,\cdot)\}_{\lambda\in[0,1]}$,
which under exchangeability yields
$\E[L(\hat\lambda;X,Y)]\le\alpha$. Claim~(i) is
Lemma~\ref{lem:commute-refine} applied with $\sigma=\sigma_{\hat\lambda}(X)$,
and it identifies $\bigcup_{c\in D}\finePart|_{c}$ as the exact support of the
discrepancy between the sparse and dense pipelines, which is the region whose
mass (ii) bounds. The tail bound in (iii) is Markov's inequality applied to the
non-negative variable $L(\hat\lambda;X,Y)$, and the characterization of its
zero set is Lemma~\ref{lem:commute-seg}.
The discarded mass fraction is precisely the loss at the calibrated threshold:
\[
  L(\hat\lambda;X,Y)
  = \frac{\sum_{c\in D}w_c(Y)}{\sum_{c\in\coursePart}w_c(Y)},
  \qquad D=\coursePart\setminus S_{\hat\lambda}(X),
\]
with $0/0:=0$. By Lemma~\ref{lem:crc-hyp}, $L(\hat\lambda;X,Y)\ge 0$, and by (ii),
$\E[L(\hat\lambda;X,Y)]\le\alpha$. For $t\in(0,1]$, Markov's inequality gives
\[
  \Pr\bigl(L(\hat\lambda;X,Y)\ge t\bigr)
  \;\le\; \frac{\E[L(\hat\lambda;X,Y)]}{t}
  \;\le\; \frac{\alpha}{t},
\]
and therefore
\[
  \Pr\bigl(L(\hat\lambda;X,Y)\le t\bigr)
  \;\ge\; \Pr\bigl(L(\hat\lambda;X,Y)< t\bigr)
  \;=\; 1-\Pr\bigl(L(\hat\lambda;X,Y)\ge t\bigr)
  \;\ge\; 1-\frac{\alpha}{t}.
\]
That is, with probability at least $1-\alpha/t$, the cells in $D$ carry at most
a $t$-fraction of the total signal mass $\sum_{c\in\coursePart}w_c(Y)$.

For the zero set, Lemma~\ref{lem:commute-seg} holds for every
$\lambda\in[0,1]$, so it applies to each realisation of $\hat\lambda$:
\[
  L(\hat\lambda;X,Y)=0
  \iff \mathcal{O}(Y)\subseteq S_{\hat\lambda}(X).
\]
This event depends on $A$ only through $S_{\hat\lambda}(X)$. Its form is
therefore the same for every admissible $A$, although the event itself can
still vary with $A$ through $\hat\lambda$.
\end{proof}

\section{Experiment Details}

\subsection{MSD Brain Tumor Segmentation}
\label{appendix:decathlon}

The MSD experiment's data, model, and calibration details are presented in this
section.

\subsubsection{Dataset Details}

We use the MSD \emph{Task01: Brain Tumour} dataset
~\citep{MedDecathlon}, comprising 484 multimodal brain MRI scans, each with four
co-registered $240\times240\times155$-voxel sequences (FLAIR, T1-weighted, contrast-
enhanced T1-weighted, and T2-weighted) and a voxel-wise annotation into three positive tumor subregions (edema, non-enhancing tumor, and enhancing tumor) plus background. The official MSD test split ships without public labels, so, following standard practice, we instead use MONAI's deterministic 80/20 split of the released
\emph{training} set: the resulting 96-scan validation subset is held out entirely
and used as our test set, while the remaining 388-scan training subset is further split
into a 320-scan model-fitting set and a 68-scan (17.5\%) calibration set, sampled with a
fixed random seed independent of the 80/20 split. The calibration set is held out of all
first-stage model training; it is used exclusively to calibrate the conformal risk-control
threshold $\hat\lambda$.

\subsubsection{Data Preprocessing}

Each modality is independently normalized by clipping its non-zero voxels to their own
99.5th percentile and rescaling to $[0, 1]$; background (zero) voxels are left untouched.
Volumes are then spatially padded to the nearest multiple of the coarse downsampling
factor $r=4$ on every axis, so
that the coarse pass's $4\times4\times4$ average/max pooling divides the volume exactly.
The coarse (low-resolution) image is obtained by trilinear downsampling of the padded
volume by $r$; the corresponding coarse label is obtained by non-overlapping max-pooling
over the same $4\times4\times4$ blocks, so that a positive (tumor) voxel is never
averaged away at low resolution. For training the first-stage (coarse and full-resolution)
segmentation models, we sample fixed $64\times64\times32$-voxel crops per scan, biased
1:1 toward foreground-containing regions, followed by random flips along each of the
three spatial axes ($p=0.5$ each) and a random $90^\circ$ in-plane rotation ($p=0.5$),
then pad back to the fixed crop size. Evaluation, calibration, and the refinement stage
instead operate on the full padded volume via sliding-window inference with the same
$64\times64\times32$ window.

\subsubsection{Model Architecture}

We compare three segmentation architectures, all trained from scratch (no pretrained
weights) and used identically as the coarse-pass model and as the full-resolution baseline -- only the spatial resolution of the input differs between the two, isolating
the effect of input resolution from any change in model capacity.

\begin{itemize}
\item \textbf{U-Net}~\citep{unet}: a 3-level 3D U-Net with channel widths $(16, 32, 64)$ and stride-2
  downsampling at each level.
\item \textbf{Swin-UNETR}~\citep{swin-unet, swinunetr-v2}: a Swin Transformer encoder with a convolutional U-Net-style
  decoder, at its default configuration.
\item \textbf{VISTA3D}~\citep{vista3d}: a segmentation foundation-model architecture composed of a
  SegResNet-based image encoder (downsampling blocks $(1,2,2,4,4)$, instance
  normalization, 48-channel embedding) and a class-mapping classification head. We run
  VISTA3D in its automatic (class-vector) segmentation mode only, without any point
  prompts, so that it can be invoked identically to the other two architectures.
\end{itemize}

The region-refinement stage uses a separate, lightweight two-branch 3D convolutional
network, common to all three first-stage architectures. Each selected coarse cell is a
$4\times4\times4$ block; adding a one-voxel halo on every side gives a
$6\times6\times6$ high-resolution window. One branch encodes this high-resolution image
window and the other encodes the upsampled coarse-model mask over the same window (two
$3\times3\times3$ convolutions, 16 channels, GELU activations); the two branches
are fused by addition and passed through two further convolutions before a $1\times1\times1$
output head. The refiner's full $6\times6\times6$ output is center-cropped to the
central $4\times4\times4$ block, which is what is scored against the ground truth and
scattered back into the full-resolution prediction; the halo voxels serve only as
context for the convolutions.

\subsubsection{Training and Calibration Details}

Each of the three first-stage architectures is trained at both full and $4\times$-downsampled
resolution, for three random seeds, with Adam (learning rate $10^{-4}$, batch size 8, up
to 100 epochs), minimizing an equally weighted sum of Dice loss and cross-entropy. A
single region-refinement network, shared by all architectures and seeds, is trained for one
epoch with Adam (learning rate $10^{-3}$, batch size 256), minimizing cross-entropy on the
central $4\times4\times4$ block only. Its training windows are the cells the
coarse model selects at its $\alpha=0.1$ calibrated threshold on the 388-scan training
subset, with that model's softmax as the coarse-mask input. CRC is calibrated on each architecture's coarse-pass
foreground probability, per training seed, using the missed-positive-cell-fraction loss
of Eq.~\ref{eq:ca-loss} on the 68-scan calibration set described above; the calibrated
threshold $\hat\lambda$, the resulting retained cell fraction, and downstream DSC and FLOPs are then evaluated on the fully held-out 96-scan test set.

\subsubsection{Qualitative Examples}
\label{appendix:decathlon-qualitative}

Fig.~\ref{fig:seg-contour-grid} in the main text (Sec.~\ref{sec:seg}) compares predicted
tumor boundaries directly; Fig.~\ref{fig:seg-threshold-sweep} here shows, for the same
three test scans, how the region CRC sends to refinement (Eq.~\ref{eq:crc})
shrinks toward the tumor as the target miss rate $\alpha$ grows, together with the
resulting measured per-scan cell miss rate. Thus, the coverage guarantee behind
Table~\ref{tab:model_comparison}'s cell miss rate column is visible at the level of a single slice,
not only as an aggregate statistic.

\begin{figure}[h]
    \centering
    \includegraphics[width=0.8\linewidth]{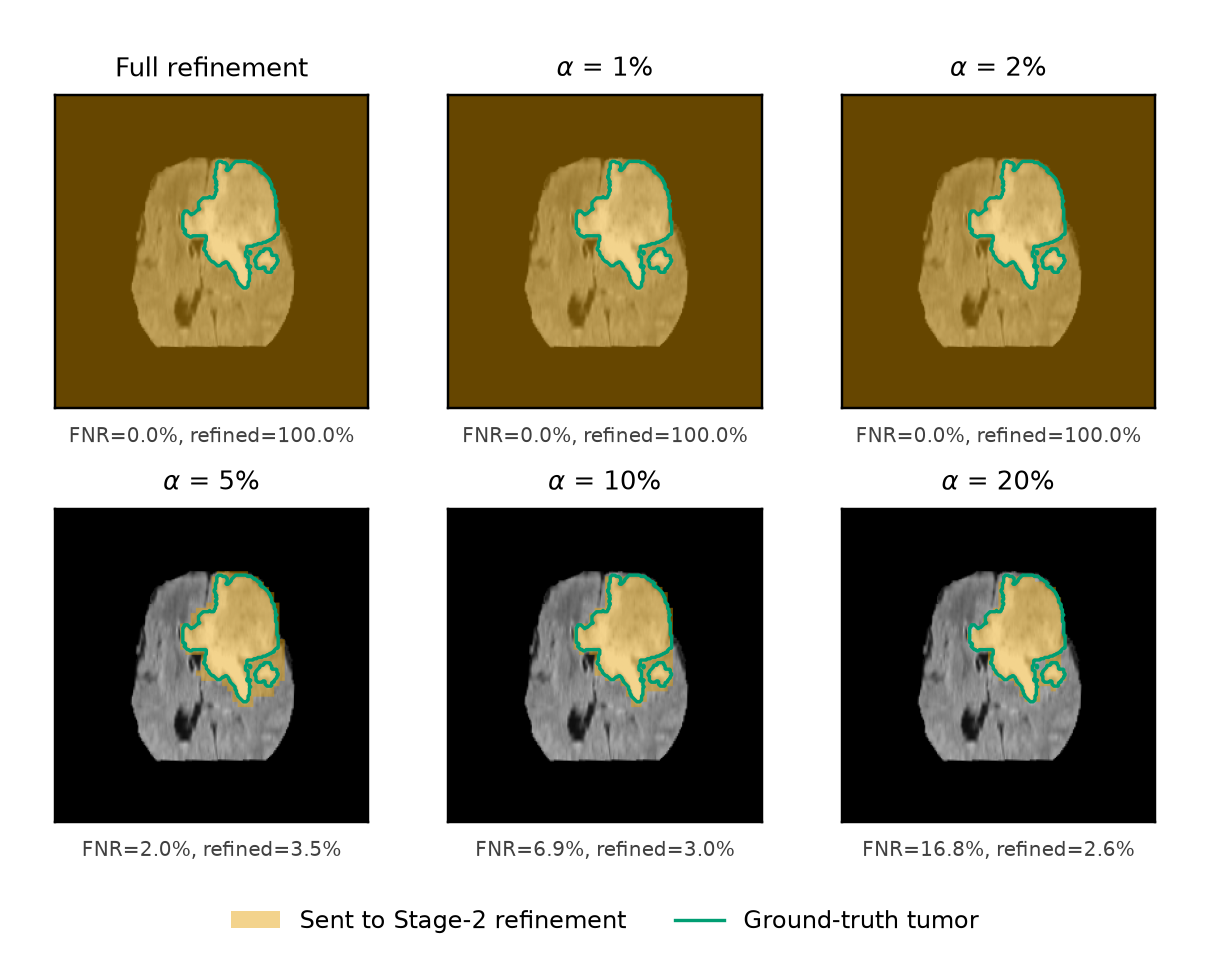}
    \\[0.4em]
    \includegraphics[width=0.8\linewidth]{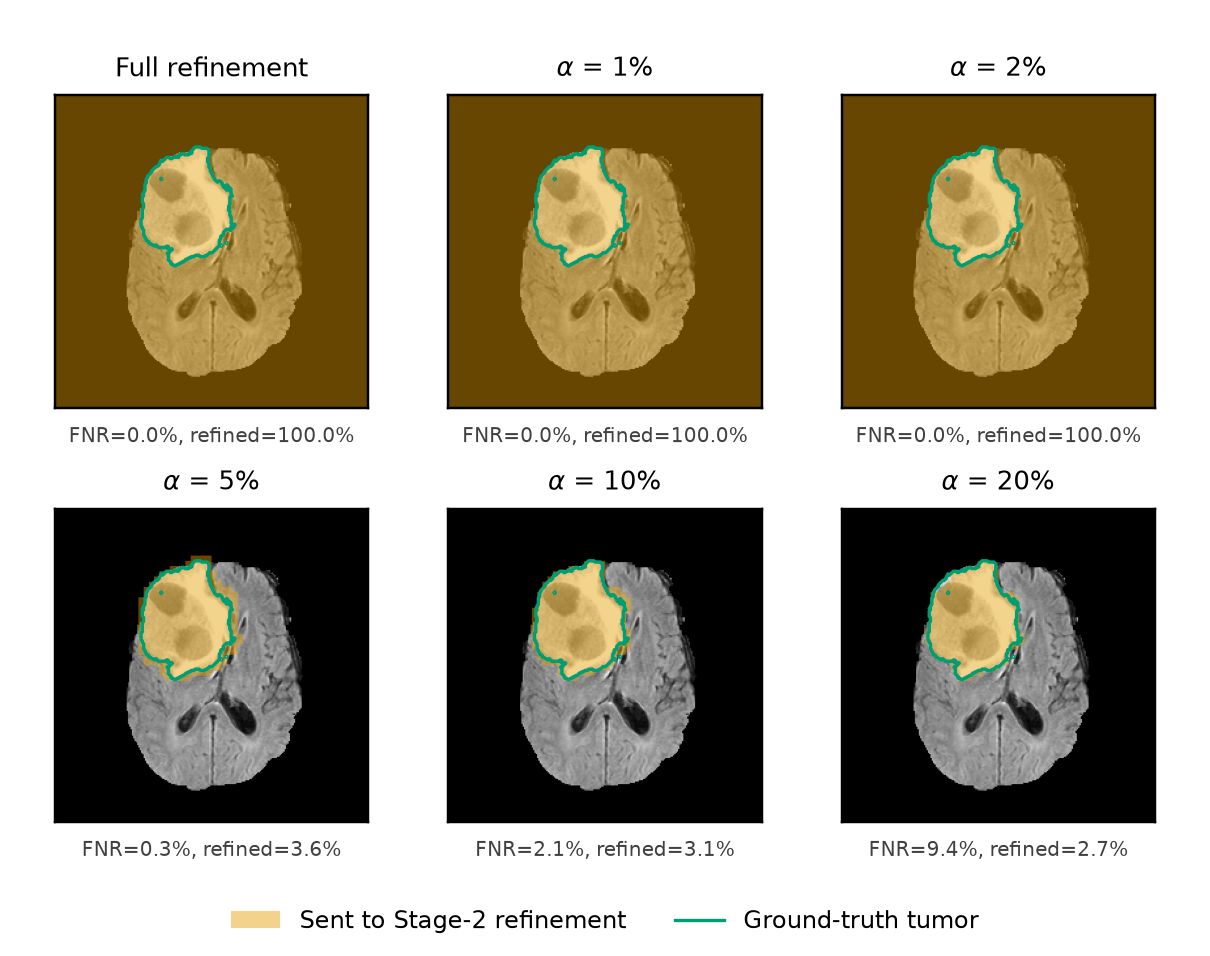}
    \caption{Selected (refined) region against the calibrated threshold for two held-out
    test scans (Swin-UNETR, seed 0). "Full refinement" is the reference at foreground-probability
    threshold $0$, i.e.\ $\lambda=1$ in Eq.~\ref{eq:mask} (every foreground-probable cell
    refined); at $\alpha=1\%$ and $2\%$, CRC's calibrated threshold is itself $0$
    ($\hat\lambda=1$) for this architecture/seed.
    As $\alpha$ relaxes, the selected region shrinks
    toward the tumor boundary and the measured cell miss rate rises toward its target.}
    \label{fig:seg-threshold-sweep}
\end{figure}

\subsection{HECKTOR 2026 Head and Neck Tumor Segmentation}
\label{appendix:hn}

\subsubsection{Data Preprocessing}
There are three NIfTI volumes per case: the original CT, the FDG-PET, and the segmentation label. The labels are included as 0 for background, 1 for GTVp (primary), and 2 for GTVn (node). 

The 782 cases are split across 531 train, 125 validation, and 126 test samples. 
The 125-case validation split is the \textbf{CRC calibration set}, not a model-selection set. Note that validation splits are obtained randomly from the training set for the Coarse and Refine models.

\subsubsection{Results}
Full test-set results for the HNC study are provided in Tables~\ref{tab:hn-results} and~\ref{tab:hn-lesion-results}. 
It is important to note that the there is a drop in accuracy especially for GTVn predictions when using the Coarse + Refine mechanism. This is particularly due to the fact that there small lesions that are being missed by the coarse model and being enabled by the CRC loss on cell-level misses. We acknowledge this as a negative result for our method as it is a particularly adversarial task given many small lesions in HNCs. A custom designed coarse model and a lesion-based CRC loss can potentially significantly improve our current lesion-based results as highlighted before, is a significant area of future work that can be pursued. 
\begin{table}[h]
\centering
\caption{DSC and FLOPs speedup across models. All rows except nnU-Net are Swin-UNETR. Mean DSC is the unweighted average of GTVp and GTVn; speedup is the ratio of Full-model FLOPs to pipeline FLOPs.}
\label{tab:hn-results}
\begin{tabular}{lcccr}
\toprule
System & GTVp & GTVn & Mean  & Speedup \\
\midrule
nnU-Net (250 ep) & \textbf{0.7938} & \textbf{0.7271} & \textbf{0.7604}  & -- \\
\midrule
Full & 0.7755 & 0.6957 & 0.7356  & 1.0$\times$ \\
Coarse & 0.7487 {\scriptsize $\pm$ 0.0117} & 0.5660 {\scriptsize $\pm$ 0.0307} & 0.6573 {\scriptsize $\pm$ 0.0145} & 68.3$\times$ \\
\midrule
Coarse + Refine ($\alpha{=}0.1$) & 0.7810 {\scriptsize $\pm$ 0.0005} & 0.6646 {\scriptsize $\pm$ 0.0068} & 0.7228 {\scriptsize $\pm$ 0.0034} & 68.1$\times$ \\
\bottomrule
\end{tabular}
\end{table}

\begin{table}[h]
\centering
\caption{Lesion-specific results for our HNC models. All rows except nnU-Net are Swin-UNETR. Unlike the volumetric DSC, we perform post-processing to detect lesions from predicted volumes to calculate lesion detection metrics.}
\label{tab:hn-lesion-results}
\begin{tabular}{lcccccc}
\toprule
& \multicolumn{3}{c}{GTVp} & \multicolumn{3}{c}{GTVn} \\
\cmidrule(lr){2-4}\cmidrule(lr){5-7}
System & Recall & Prec. & DSC & Recall & Prec. & DSC \\
\midrule
nnU-Net (250 ep) & \textbf{0.766} & 0.572 & 0.819 & \textbf{0.615} & 0.378 & \textbf{0.773} \\
\midrule
Full & 0.685 & 0.625 & \textbf{0.822} & 0.553 & \textbf{0.526} & 0.772 \\
Coarse & 0.599 & 0.628 & 0.799 & 0.339 & 0.471 & 0.745 \\
\midrule
Coarse + Refine ($\alpha{=}0.1$) & 0.680 & \textbf{0.638} & 0.814 & 0.465 & 0.475 & 0.732 \\
\bottomrule
\end{tabular}
\end{table}

Qualitative results for the predictions are presented in Fig.~\ref{fig:hecktor-predictions}.
\begin{figure}[h]
    \centering
    \includegraphics[width=0.95\linewidth]{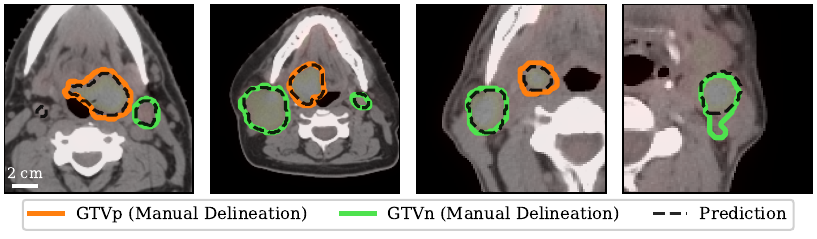}
    \includegraphics[width=0.95\linewidth]{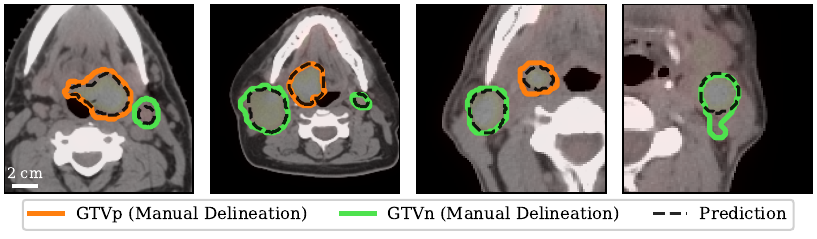}
    \includegraphics[width=0.95\linewidth]{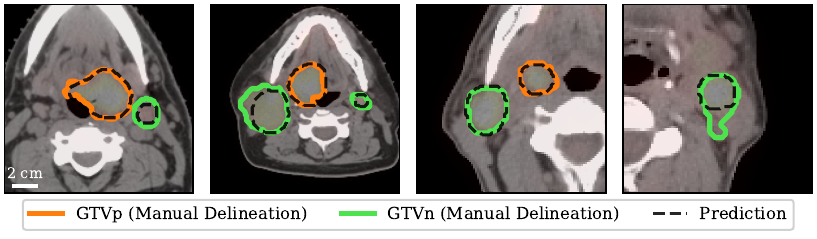}
    \includegraphics[width=0.95\linewidth]{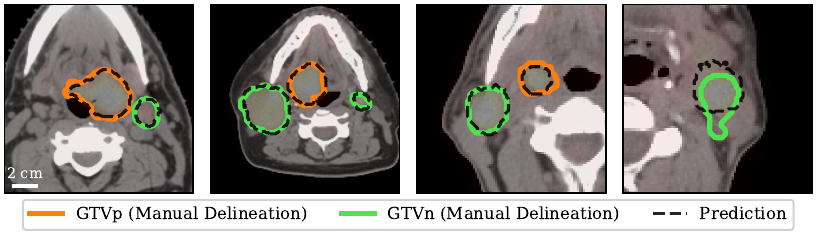}
    \caption{Grid of HECKTOR predictions. From top to bottom: nnU-Net, Swin-UNETR native resolution, Swin-UNETR Coarse, and Swin-UNETR Coarse + Refine.}
    \label{fig:hecktor-predictions}
\end{figure}
\subsection{Threshold selection for Segmentation}
\label{appendix:threshold}
Depending on the model and dataset, the threshold can result in significantly varied FNR as seen in Tab~\ref{tab:fnrvsthreshold}. This highlights the need for using an established threshold finding algorithm like CRC. Naively choosing $\lambda$ can either result in selecting too many cells, diminishing the speedup achieved, or too few cells, resulting an underestimated error. Furthermore, as Tab.~\ref{tab:fnrvsthreshold} shows, different models and tasks have different thresholds. Thus effectively selecting a threshold requires a held-out calibration set, and CRC is a natural algorithm for choosing such a threshold. Also, using CRC allows us to expand our choices beyond solely threshold based methods. The pre-requisite for using CRC is monotonically increasing prediction sets which can be accomplished by other selection criterion as well. 
\begin{table}[H]
    \centering
    \caption{A naively chosen like threshold can result in high FNR as the predicted scores are non-linear compared to the selection. CRC }
    \label{tab:fnrvsthreshold}
    \begin{minipage}{\linewidth}
    \centering
    \resizebox{\linewidth}{!}{%
    \begin{tabular}{llccccc}
    \toprule
    & & \multicolumn{3}{c}{FNR at fixed $\lambda$} & \multicolumn{2}{c}{FNR with CRC ($\alpha = 0.1$)} \\
    \cmidrule(lr){3-5} \cmidrule(lr){6-7}
    Dataset & Model & $\lambda = 0.5$ & $\lambda = 0.1$ & $\lambda = 0.01$ & Calib. & Test \\
    \midrule
    \multirow{3}{*}{Brain Tumor} & UNet & 0.17 {\scriptsize $\pm$ 0.01} & 0.072 {\scriptsize $\pm$ 0.008} & 0.021 {\scriptsize $\pm$ 0.003} & 0.085 {\scriptsize $\pm$ $<$0.001} & 0.097 {\scriptsize $\pm$ 0.006} \\
     & Swin-UNETR & 0.11 {\scriptsize $\pm$ 0.02} & 0.07 {\scriptsize $\pm$ 0.01} & 0.031 {\scriptsize $\pm$ 0.006} & 0.086 {\scriptsize $\pm$ $<$0.001} & 0.091 {\scriptsize $\pm$ 0.001} \\
     & VISTA3D & 0.11 {\scriptsize $\pm$ 0.01} & 0.07 {\scriptsize $\pm$ 0.01} & 0.04 {\scriptsize $\pm$ 0.01} & 0.086 {\scriptsize $\pm$ 0.001} & 0.084 {\scriptsize $\pm$ 0.004} \\
    \midrule
    HECKTOR & Swin-UNETR & 0.86 {\scriptsize $\pm$ 0.02} & 0.25 {\scriptsize $\pm$ 0.02} & 0.13 {\scriptsize $\pm$ 0.01} & 0.092 {\scriptsize $\pm$ 0.001} & 0.072 {\scriptsize $\pm$ 0.009} \\
    \bottomrule
    \end{tabular}%
    }
    \end{minipage}
\end{table}
\subsection{CASCADE Calcium Spike Prediction}
\label{appendix:CASCADE}
The CASCADE experiment data processing, model design, and calibration methods are presented in this section. 

\subsubsection{Dataset Details}

We use the \emph{primary} group of the CASCADE calcium-imaging benchmark: datasets DS30, DS31,
and DS32, recorded with the GCaMP8f/8m/8s indicators respectively. Each dataset file contains
one or more trials, where a trial pairs a fluorescence ($\Delta F/F$) trace with ground-truth
spike times obtained from simultaneous electrophysiological recording; frame rates vary by
recording session, ranging from 8--60\,Hz across the benchmark. The primary group comprises 117
source \texttt{.mat} files spanning 515 valid trials (trials with fewer than 2 recorded spike
events or fewer than 10 fluorescence samples are discarded), at a mean spike-mass occupancy
(fraction of non-empty coarse blocks at $r{=}8$) of $3.2$--$4.0\%$.

Because repeated trials from the same source file are highly dependent draws from the same
underlying neuron (identical indicator, noise level, and firing statistics), we split
\emph{source files}, not individual windows or trials, into train/calibration/test partitions
at a $60/20/20$ ratio, stratified per dataset so all three indicators appear in every partition.
This file-level split is the exchangeable unit for the CRC
calibration; every downstream step (coarse model training, CRC calibration, refiner training,
evaluation) consumes the same fixed seed-0 split, and fluorescence windows are unique between splits. This split rules out test-set leakage; the standard sample-level calibration algorithm is then applied unchanged. 

\subsubsection{Data Preprocessing}

Each trial's raw spike times are binned onto the trial's own fluorescence frame
timestamps via a histogram, and the resulting spike-count sequence is convolved with a Gaussian
kernel (standard deviation $200$\,ms, the CASCADE benchmark's own default) and rescaled to spikes/second
to obtain the fine-resolution regression target $Y^{\finePart}$. The coarse target
$Y^{\coursePart}$ is obtained by \emph{sum}-pooling (not averaging) the raw spike counts by
the coarse block factor $r{=}8$, so that each coarse cell's value is exactly the expected spike
count inside it, matching the aggregation operator of Eq.~\ref{eq:weight} for regression. The
CRC loss is the fraction of missed spike-count mass. The fluorescence trace
itself is $z$-score normalized per source file (median and $1.4826\times$MAD, computed once across
all of a file's trials to avoid leaking window-local statistics), then windowed into
fixed-length segments of $1024$ fine samples ($128$ coarse blocks at $r{=}8$), with stride $512$
for training, calibration, and test.

\subsubsection{Model Architecture}
Both the baseline and coarse networks in the pipeline are 1D convolutional nets sharing a common building block
(\texttt{ConvBlock}: two $\mathrm{Conv1d}(k{=}3)$+GroupNorm+GELU layers) and a common output
head design: a \texttt{QuantileHead} predicting $Q{=}4$ quantiles
($\tau\in\{0.05,0.5,0.9,0.95\}$) by regressing the $\tau{=}0.05$ quantile plus non-negative
(softplus) increments, cumulatively summed to guarantee monotonicity.

\begin{itemize}
\item \textbf{Full baseline $\mathcal{F}_{\finePart}$} (\texttt{FineUNet1D}): a $4$-level 1D
  U-Net, base width $64$, with $\mathrm{MaxPool1d}(2)$ downsampling and nearest-neighbor
  upsample + skip-concatenation on the decoder side, applied to the full $1024$-sample window.
\item \textbf{Coarse model $\mathcal{F}_{\coursePart}$} (\texttt{FineOnLowRes}): reuses
  $\mathcal{F}_{\finePart}$'s architecture and parameter count \emph{exactly}, but is fed the trace
  block-averaged down to $\mathrm{window}/r=128$ samples before the first convolution, rather
  than a separate, narrower network. This isolates the coarse model's speedup to resolution
  reduction alone: at identical capacity to $\mathcal{F}_{\finePart}$, this gives an $8.0\times$ FLOPs
  reduction for the coarse pass (exactly $r$, since a fixed-width conv net's FLOPs scale
  linearly with sequence length), so no reviewer can attribute the pipeline's speedup to simply
  swapping in a smaller-capacity model.
\item \textbf{Refiner $\cG_{\coursePart\to\finePart}$} (\texttt{LocalRefinerUNet1D}): a $2$-level 1D U-Net
  (base width $32$) over a context-padded crop of length $L=r(1+2\times\mathrm{context\_blocks})
  = 136$ samples ($\mathrm{context\_blocks}{=}8$) around each retained coarse block, whose
  full-length decoder output is center-cropped to the middle $r{=}8$ samples. Rather than
  predicting the fine-resolution rate curve $Y^{\finePart}$ directly (which would hard-zero
  discarded neighboring blocks regardless of proximity to a confirmed detection), the refiner
  uses a \texttt{CountHead} (global-average-pool + linear projection) to predict a single
  non-negative spike-\emph{count} estimate per retained block, matching $Y^{\coursePart}$'s
  own resolution. The full continuous rate curve is then
  reconstructed by convolving this sparse, mostly-zero per-block count estimate with the same
  Gaussian kernel used to build $Y^{\finePart}$ from raw spike counts. Since that kernel is
  mass-preserving, bounding missed count mass at the CRC level $\alpha$ (calibrated against
  $Y^{\coursePart}$) also bounds the missed mass of the reconstructed continuous curve at the
  same $\alpha$.  A detection's estimated count can be smeared physically into its discarded
  neighbors rather than being hard-zeroed there during post processing and presentation. \textbf{Note that the certification region is decoupled from the presentation in high resolution.}
\end{itemize}

Table~\ref{tab:cascade-arch} summarizes parameter counts and analytic FLOPs (at
$\mathrm{window}{=}1024$) for the coarse model and full-resolution baseline, computed with each model's default constructor arguments.

\begin{table}[h]\centering
\caption{CASCADE model architectures, parameter counts, and analytic FLOPs. The total FLOPs for the refiner depends on the number of selected windows and is therefore omitted.}
\label{tab:cascade-arch}
\resizebox{\textwidth}{!}{
\begin{tabular}{lccc}\toprule
Model & Role & Params & FLOPs (window$=$1024) \\ \midrule
\texttt{FineUNet1D} & Full baseline ($\mathcal{F}_{\finePart}$) & 10{,}475{,}204 & 2{,}668{,}494{,}848 \\
\texttt{FineOnLowRes} & Coarse ($\mathcal{F}_{\coursePart}$) & 10{,}475{,}204 (same) & 333{,}561{,}856 (8.0$\times$ fewer) \\
\texttt{LocalRefinerUNet1D} & Refiner (count head) & 158{,}721 & -- \\ \bottomrule
\end{tabular}}
\end{table}

\subsubsection{Training and Calibration Details}

All three networks (baseline, coarse model, refiner) are trained with Adam (learning rate
$10^{-3}$) for $30$ epochs, minimizing the quantile pinball loss summed over
$\tau\in\{0.05,0.5,0.9,0.95\}$ for the baseline and coarse model, and mean-squared error against
the block spike count for the refiner's count head. CRC is calibrated on the coarse model's
predicted $\tau{=}0.9$ quantile score field (empirically the best-performing of the median/q90/
q95 candidates by screening ROC AUC), using the missed-signal-mass loss of
Eq.~\ref{eq:ca-loss}, with a retention halo of $1$ coarse block on either side of every
selected block (chosen heuristically).
For the reported operating point ($\alpha{=}0.1$), CRC yields a calibrated $\hat\lambda$ with a raw-score cutoff of $0.877$ and a
retained fraction of $8.4\%$ of coarse blocks on held-out test data, corresponding to the
$4.77\times$ FLOPs speedup reported in Table~\ref{tab:cascade-headline}. A separate refiner is
trained per calibrated $\alpha$ (each refiner's training crops are drawn from that $\alpha$'s
own retention distribution), and reconstructed-curve accuracy is reported as mean absolute
error (MAE) against $Y^{\finePart}$ on the test split, compared against the full-resolution
baseline's own median ($\tau{=}0.5$) prediction on the same windows.

\subsubsection{Results}
The reconstruction results with varying $\alpha$, their corresponding retained fraction, and final error is shown in Tab.\ref{tab:cascade-headline}. 
\begin{table}[h]\centering
\caption{Reconstructed MAE is reported in Hz, with the gap to the full-resolution baseline ($0.250$\,Hz) in parentheses.}
\label{tab:cascade-headline}
\begin{tabular}{lccc}\toprule
$\alpha$ & Retained fraction & FLOPs speedup & Reconstructed MAE ($\Delta$) \\ \midrule
0.01 & 1.00 & 0.89$\times$ & 0.30 ($+0.05$) \\
0.05 & 0.48 & 1.65$\times$ & 0.34 ($+0.09$) \\
0.10 & 0.08 & 4.77$\times$ & 0.32 ($+0.06$) \\
0.20 & 0.07 & 5.21$\times$ & 0.35 ($+0.10$) \\ \bottomrule
\end{tabular}
\end{table}

\subsection{Galaxy Flux Field Super-Resolution}
\label{appendix:galsim}
The galaxy flux field experiment's data generation, model design, training, and calibration methods are
presented in this section.

\subsubsection{Dataset Details}

Each sample is a simulated $1024\times1024$ sky field at a pixel scale of $0.2''$, covering roughly $3.4'\times3.4'$. The number of sources per field is Poisson-distributed with mean $60$, so the realized count varies field to field and occupancy is a statistic over fields rather than a fixed per-field quantity. We generate $3000$ fields in total, partitioned $2000/500/500$ into train, calibration, and test splits. One field is one exchangeable unit for conformal risk control: unlike the calcium-imaging study there is no windowing, so a field is never split
across partitions.

Sources are real galaxies drawn from the COSMOS $25.2$ magnitude-limited sample ($81{,}499$ objects) rather than point sources. We use the analytic S\'{e}rsic
fits to the catalog objects rather than the raw telescope pixel data, so each source carries the real size, shape, orientation, and S\'{e}rsic-index distribution of an observed galaxy but is a smooth, noiseless profile. 

We avoid the raw pixel data because it carries the original exposure's own correlated noise: rescaling a source to this experiment's flux prior can require a factor of more than four orders of magnitude, which would amplify that embedded noise by the same factor and introduce spatially non-stationary contamination unrelated to the sky and read-noise model the experiment is built around.

The three splits draw from \emph{disjoint contiguous slices} of the catalog, in the proportions of the split sizes. We ensure the same galaxy shape (same catalog object, differing only in flux, position, rotation, and noise) cannot appear in both a training field and a test field, since every field otherwise samples uniformly from the whole catalog. This is a known failure mode for this catalog, and holding out objects rather than only images follows standard practice for COSMOS-based generative and deblending studies~\citep{arcelin2021deblending}.

\subsubsection{Data Generation Pipeline}

Every field is reproducible from a single integer seed; all randomness is drawn from one seeded generator, so caching amounts to calling the generator once per seed. For each field we draw the source count and then, for each source, draw a flux log-uniformly on $[10^{2},10^{5}]$ counts, a position uniform in the field interior (a $20$-pixel margin is excluded so sources are not truncated at the boundary), a catalog object from that split's slice, and a uniform rotation. Each source is rendered into a $65\times65$ pixel stamp centered on its position.

The same object is rendered \emph{twice}, in two different roles, and the two blurring operators
are kept strictly separate:
\begin{itemize}
\item \textbf{Input image.} The source is convolved with the instrument point spread function, a
  Kolmogorov profile of full width at half maximum $0.8''$, i.e.\ the forward model of a real
  observation.
\item \textbf{Regression target.} The same source is convolved instead with a small idealized
  Gaussian kernel with a standard deviation of $1.5$ pixels ($0.3''$), and no noise is added.
\end{itemize}
Neither kernel is derived from the other. The target is therefore a deconvolved,
noise-free surface-brightness map of the same scene: the task is simultaneously deblurring and
denoising, not merely upsampling.

After all sources are drawn, synthetic detector noise is added to the input image only, as Poisson
noise about a sky level of $50$ counts plus Gaussian read noise of $5$ counts at unit gain. This is
the only noise in the experiment. The known sky level is subtracted from the image before it is
presented to any network, so a network's input is a background-subtracted, noisy, PSF-convolved
field and its target is the same field's noiseless, sharply rendered flux.

The coarse partition is the $8\times8$-pixel cell grid, giving $128\times128=16{,}384$ coarse
cells per field. The coarse target is obtained by \emph{sum}-pooling the fine target over each
cell, not averaging, so a cell's value is exactly the total flux inside it and matches the
regression aggregation operator of Eq.~\ref{eq:weight}; the CRC loss is then literally the
fraction of missed flux. The coarse model's input is obtained by \emph{average}-pooling the observed
image by the same factor, because the image is a surface brightness rather than a total.

\subsubsection{Model Architecture}

All three networks share a common convolutional block (two $3\times3$ convolutions, each with
group normalization and a GELU nonlinearity) and, for the two predicting networks, a common
output head predicting $Q=4$ quantiles ($\tau\in\{0.05,0.5,0.9,0.95\}$) of the log-transformed
flux. The head regresses the $\tau{=}0.05$ quantile plus non-negative softplus increments,
cumulatively summed, so predicted quantiles are monotone by construction and can never cross.
All flux regression is performed in $\log(1+\cdot)$ space, which is necessary given the
three-decade flux prior.

\begin{itemize}
\item \textbf{Full baseline.} A four-level two-dimensional U-Net of base width $64$, with
  max-pool downsampling and nearest-neighbor upsampling with skip concatenation, applied to the
  full $1024\times1024$ field.
\item \textbf{Coarse pass.} Reuses the full-resolution baseline's architecture and parameter
  count \emph{exactly}, applied to the $8\times$ average-pooled $128\times128$ field, rather than a
  separate narrower network. This isolates the coarse pass's savings to resolution alone: at
  identical capacity it is a $64\times$ reduction in floating-point operations (exactly the square
  of the downsampling factor, since a fixed-width convolutional network's cost scales with the
  number of sites), so the speedup cannot be attributed to substituting a smaller-capacity model.
  Its output is one predicted flux total per coarse cell.
\item \textbf{Refiner.} A small local operator applied independently to each certified cell,
  taking the $12\times12$ full-resolution window formed by the $8\times8$ cell plus a two-pixel
  context border, together with the coarse model's own predicted median flux for that cell as a
  scalar conditioning input, and returning an $8\times8$ flux patch. Rather than re-predicting
  \emph{how much} flux the cell contains, which the coarse model already estimates in a calibrated
  way, the refiner predicts only \emph{where within the cell} the flux lies. Its output is
  rectified, normalized to sum to one over the cell, and then multiplied by the coarse model's
  predicted total. Flux conservation between the coarse and refined stages is therefore exact by
  construction rather than something the refiner must learn, and the refiner is left with the
  purely spatial distribution of the problem, which is what makes a network of $17{,}809$ parameters
  sufficient. Uncertified cells are filled by smoothly upsampling the coarse prediction, so a selection miss degrades to coarse-model quality rather than to blank sky.
\end{itemize}

Table~\ref{tab:galsim-arch} summarizes parameter counts and analytic operation counts for the
three networks.

\begin{table}[h]\centering
\caption{Galaxy flux field model sizes and analytic floating-point operation counts, per $1024^2$ field
for the full-resolution baseline and the coarse model, and per certified cell for the refiner.}
\label{tab:galsim-arch}
\resizebox{\textwidth}{!}{
\begin{tabular}{lccc}\toprule
Model & Role & Params & FLOPs \\ \midrule
2D U-Net, base width 64 & Full baseline & 31{,}389{,}764 & $1.78\times10^{12}$ \\
Same U-Net, $8\times$-downsampled input & Coarse & 31{,}389{,}764 (same) & $2.78\times10^{10}$ ($64\times$ fewer) \\
Local refinement operator & Refiner & 17{,}809 & $3.05\times10^{6}$ per cell \\ \bottomrule
\end{tabular}}
\end{table}

\subsubsection{Training and Calibration Details}

All three networks are trained with Adam at a learning rate of $10^{-3}$ for $30$ epochs on the
$2000$ training fields, holding out a fixed $10\%$ of them as an internal validation split, with
all randomness seeded. The full-resolution baseline and the coarse model minimize the quantile pinball loss
summed over $\tau\in\{0.05,0.5,0.9,0.95\}$, against the log-transformed fine target and the
log-transformed sum-pooled coarse target respectively. The refiner minimizes an $L_1$ loss in
$\log(1+\cdot)$ space between its predicted patch and the true full-resolution flux patch. 
The refiner's output is anchored to the coarse model's calibrated total.

The refiner is trained on the \emph{certified} distribution rather than on uniformly sampled cells. We run the deployed, already-calibrated coarse model over the training split, apply the calibrated threshold and halo, and draw the refiner's training crops from the resulting certified regions; this yields $716{,}629$ cells over $2000$ fields, $2.19\%$ of all cells. Halo cells are deliberately included, since they are near-empty by construction and are where boundary truncation errors appear. Each cell's coarse-predicted median flux is cached at the same time.

CRC is calibrated on the $500$ calibration fields. 
The score field is the coarse model's predicted $\tau{=}0.9$ quantile of cell
flux. We use an upper quantile rather than the median so that a cell whose median flux is low but whose upper tail is high is still certified. Cell weights are flux mass, the aggregation operator of Eq.~\ref{eq:weight}, so the calibrated loss of Eq.~\ref{eq:ca-loss} is the fraction of a field's true flux that falls in uncertified cells. The certified set is dilated by a halo of one coarse cell (chosen heuristically)  in an eight-connected neighborhood before the loss is evaluated. Calibration scans an ascending grid of $500$ candidate thresholds and takes the largest whose finite-sample corrected risk $(n\hat{R}+1)/(n+1)$ remains at or below $\alpha$.

At the reported operating point $\alpha{=}0.1$, calibration returns a $\hat\lambda$ with a raw-score cutoff of
$7.69$ in $\log(1+\mathrm{flux})$ units, a flux total of approximately $2175$ counts
per cell. The calibration risk is $0.040$ ($0.042$ after the finite-sample correction), and on the
held-out test split the empirical risk is $0.040$ ($0.042$ corrected), comfortably within
$\alpha$, at a certified fraction of $2.18\%$ of coarse cells. Certifying $2.18\%$ of cells leaves
the pipeline at $61.6\times$ fewer floating-point operations than the full-resolution baseline, against
a ceiling of $64\times$ set by the coarse pass alone; the refinement of certified cell in an
average field costs about $4\%$ of what the coarse pass costs, so the coarse pass, not the
refinement, bounds the achievable speedup.

As a reference for how tight the conformal selection is, we also compute an $\alpha$-matched
oracle, the threshold applied directly to the \emph{true} coarse flux totals that attains exactly
risk $\alpha$ under the same halo. That oracle certifies $1.43\%$ of cells, so the calibrated
coarse model certifies $1.53\times$ the cells an oracle selector would need at the same risk level.

\subsubsection{Results}
We provide an assortment of predictions from our model in Fig.~\ref{fig:galaxy-sim-134}.
\begin{figure}[h]
    \centering
    \includegraphics[width=0.95\linewidth]{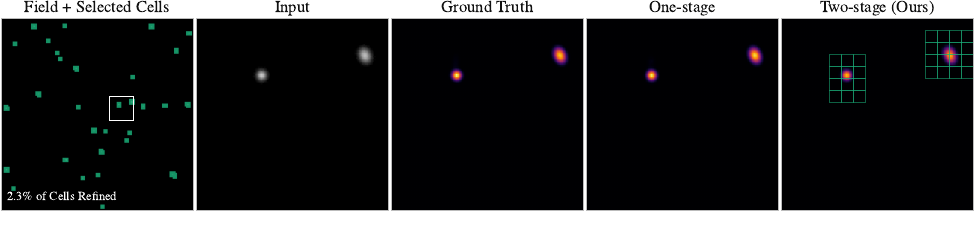}
    \includegraphics[width=0.95\linewidth]{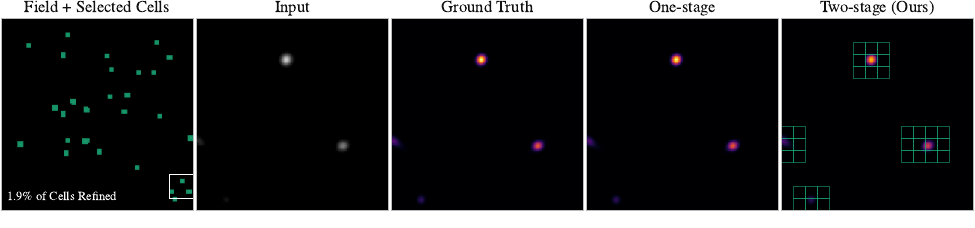}
    \includegraphics[width=0.95\linewidth]{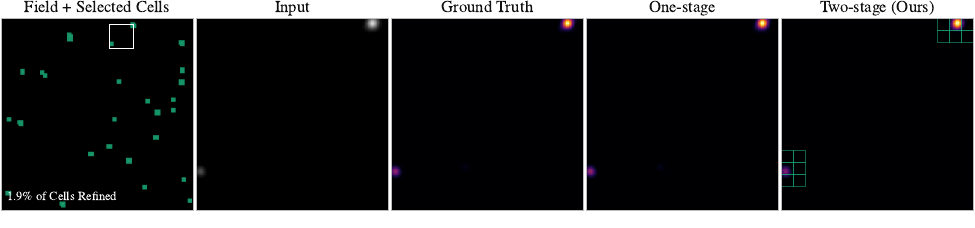}
    \includegraphics[width=0.95\linewidth]{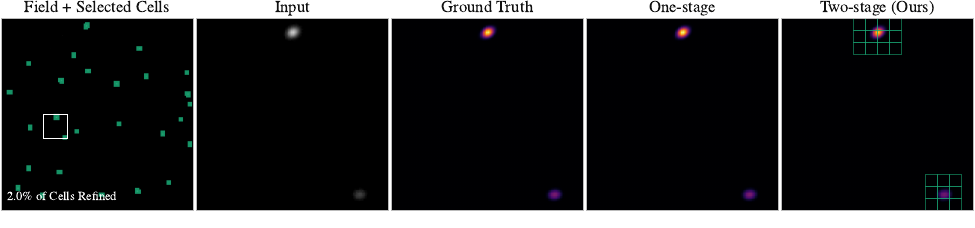}
    \caption{Grid of galaxy flux field predictions.}
    \label{fig:galaxy-sim-134}
\end{figure}

\end{document}

%% file: iclr2027/math_commands.tex
\usepackage{amsmath,amsfonts,bm}

\def\eqref#1{equation~\ref{#1}}

\def\1{\bm{1}}

\DeclareMathAlphabet{\mathsfit}{\encodingdefault}{\sfdefault}{m}{sl}
\SetMathAlphabet{\mathsfit}{bold}{\encodingdefault}{\sfdefault}{bx}{n}

\newcommand{\E}{\mathbb{E}}

